\newif\ifdraft \draftfalse
\newif\iffull \fulltrue

\documentclass[11pt]{article}

\usepackage{authblk}
\usepackage[utf8]{inputenc} 
\usepackage[T1]{fontenc}    
\usepackage{hyperref}       
\usepackage{url}            
\usepackage{booktabs}       
\usepackage{amsfonts}       
\usepackage{nicefrac}       
\usepackage{microtype}      
\usepackage{fullpage}
\usepackage{amsmath, amssymb, amsthm}
\usepackage{mathtools}
\usepackage{color}
\usepackage{kpfonts}
\usepackage{xcolor}
\usepackage{multicol}
\usepackage{algorithm}
\usepackage[noend]{algpseudocode}
\usepackage{subcaption}

\usepackage[numbers,sort&compress]{natbib}

\definecolor{DarkGreen}{rgb}{0.1,0.5,0.1}
\definecolor{DarkRed}{rgb}{0.5,0.1,0.1}
\definecolor{DarkBlue}{rgb}{0.1,0.1,0.5}
\usepackage[]{hyperref}
\hypersetup{
    unicode=false,          
    pdftoolbar=true,        
    pdfmenubar=true,        
    pdffitwindow=false,      
    pdftitle={},    
    pdfauthor={}
    pdfsubject={},   
    pdfnewwindow=true,      
    pdfkeywords={keywords}, 
    colorlinks=true,       
    linkcolor=DarkRed,          
    citecolor=DarkGreen,        
    filecolor=DarkRed,      
    urlcolor=DarkBlue,          
}
\usepackage{xspace}
\usepackage{cleveref}
\usepackage{thmtools, thm-restate}

\newcommand{\sw}[1]{\ifdraft \textcolor{blue}{[Steven: #1]}\fi}

\newcommand{\sn}[1]{\ifdraft \textcolor{red}{[Seth: #1]}\fi}
\newcommand{\bo}[1]{\ifdraft \textcolor{orange}{[Bo: #1]}\fi}

\newcommand\RR{\mathbb{R}}
\newcommand\cA{\mathcal{A}}

\newcommand\cM{\mathcal{M}}

\newcommand\cX{\mathcal{X}}

\newcommand{\Lap}[1]{\text{Lap}\left({#1}\right)}

\DeclareMathOperator*{\Expectation}{\mathbb{E}}
\newcommand{\Ex}[2]{\Expectation_{#1}\left[#2\right]}

\newcommand{\prob}[1]{\Pr\left[#1\right]}

\newcommand{\Loss}[1]{\text{Loss(}{#1})}
\newcommand{\TP}[1]{\mathcal{TP}(#1)}
\newcommand{\eps}{\varepsilon}
\def\epsilon{\varepsilon}

\DeclareMathOperator*{\argmin}{\mathrm{argmin}}
\DeclareMathOperator*{\A}{\mathcal{A}}

\newcommand{\bl}{\bullet}

\newtheorem{theorem}{Theorem}[section]
\newtheorem{lemma}[theorem]{Lemma}

\newtheorem{claim}[theorem]{Claim}
\newtheorem{remark}[theorem]{Remark}

\theoremstyle{definition}
\newtheorem{definition}[theorem]{Definition}

\newcommand{\Dnoise}{\textsc{CovNR}}
\newcommand{\Onoise}{\textsc{OutputNR}}
\newcommand{\NR}{\textsc{NR}}
\newcommand{\IAT}{\textsc{IAT}}
\newcommand{\DM}{\textsc{DoublingMethod}}

\title{Accuracy First: Selecting a Differential Privacy Level for Accuracy-Constrained ERM}
\author[1]{Katrina Ligett}
\author[2]{Seth Neel}
\author[2]{Aaron Roth}
\author[2]{Bo Waggoner}
\author[2]{Z. Steven Wu}

\affil[1]{Caltech and Hebrew University}
\affil[2]{University of Pennsylvania}

\date{May 2017}

\begin{document}

\maketitle

\begin{abstract}
Traditional approaches to differential privacy assume a fixed privacy requirement $\epsilon$ for a computation, and attempt to maximize the accuracy of the computation subject to the privacy constraint. As differential privacy is increasingly deployed in practical settings, it may often be that there is instead a fixed accuracy requirement for a given computation and the data analyst would like to maximize the privacy of the computation subject to the accuracy constraint. This raises the question of how to find and run a maximally private empirical risk minimizer subject to a given accuracy requirement.
We propose a general ``noise reduction'' framework that can apply to a variety of private empirical risk minimization (ERM) algorithms, using them to ``search'' the space of privacy levels to find the empirically strongest one that meets the accuracy constraint, incurring only logarithmic overhead in the number of privacy levels searched.
The privacy analysis of our algorithm leads naturally to a version of differential privacy where the privacy parameters are dependent on the data, which we term \textit{ex-post} privacy, and which is related to the recently introduced notion of privacy odometers. We also give an \textit{ex-post} privacy analysis of the classical AboveThreshold privacy tool, modifying it to allow for queries chosen depending on the database. Finally, we apply our approach to two common objectives, regularized linear and logistic regression, and empirically compare our noise reduction methods to (i) inverting the theoretical utility guarantees of standard private ERM algorithms and (ii) a stronger, empirical baseline based on binary search. 
\end{abstract}

\section{Introduction and Related Work}\label{sec:intro}
Differential Privacy~\cite{DMNS06,DRbook} enjoys over a decade of study as a theoretical construct, and a much more recent set of large-scale practical deployments, including by Google \cite{RAPPOR} and Apple \cite{apple}. As the large theoretical literature is put into practice, we start to see disconnects between assumptions implicit in the theory and the practical necessities of applications. In this paper we focus our attention on one such assumption in the domain of private empirical risk minimization (ERM): that the data analyst first chooses a privacy requirement, and then attempts to obtain the best accuracy guarantee (or empirical performance) that she can, given the chosen privacy constraint. Existing theory is tailored to this view: the data analyst can pick her privacy parameter $\epsilon$ via some exogenous process, and either plug it into a ``utility theorem'' to upper bound her accuracy loss, or simply deploy her algorithm and (privately) evaluate its performance. There is a rich and substantial literature on private convex ERM that takes this approach, weaving tight connections between standard mechanisms in differential privacy and standard tools for empirical risk minimization. These methods for private ERM include output and objective perturbation \cite{CMS11,DBLP:journals/jmlr/KiferST12,DBLP:journals/corr/abs-0911-5708,DBLP:conf/nips/ChaudhuriM08}, covariance perturbation \citep{AB17}, the exponential mechanism \citep{MT07, BST14}, and stochastic gradient descent \citep{BST14, DBLP:conf/nips/WilliamsM10, DBLP:journals/jmlr/JainKT12,DBLP:conf/allerton/DuchiJW13,DBLP:conf/globalsip/SongCS13}.

While these existing algorithms take a privacy-first perspective, in practice, product requirements may impose hard accuracy constraints, and privacy (while desirable) may not be the over-riding concern. In such situations, things are reversed: the data analyst first fixes an accuracy requirement, and then would like to find the smallest privacy parameter consistent with the accuracy constraint. Here, we find a gap between theory and practice.
The only theoretically sound method available is to take a ``utility theorem'' for an existing private ERM algorithm and solve for the smallest value of $\epsilon$ (the differential privacy parameter)---and other parameter values that need to be set---consistent with her accuracy requirement, and then run the private ERM algorithm with the resulting $\epsilon$. But because utility theorems tend to be worst-case bounds, this approach will generally be extremely conservative, leading to a much larger value of $\epsilon$ (and hence a much larger leakage of information) than is necessary for the problem at hand.
Alternately, the analyst could attempt an empirical search for the smallest value of $\epsilon$ consistent with her accuracy goals. However, because this search is itself a data-dependent computation, it incurs the overhead of additional privacy loss.
Furthermore, it is not \emph{a priori} clear how to undertake such a search with nontrivial privacy guarantees for two reasons: first, the worst case could involve a very long search which reveals a large amount of information, and second, the selected privacy parameter is now itself a data-dependent quantity, and so it is not sensible to claim a ``standard'' guarantee of differential privacy for any finite value of $\epsilon$ ex-ante.
In this paper, we describe, analyze, and empirically evaluate a principled variant of this second approach, which attempts to empirically find the smallest value of $\epsilon$ consistent with an accuracy requirement. We give a meta-method that can be applied to several interesting classes of private learning algorithms and introduces very little privacy overhead as a result of the privacy-parameter search. Conceptually, our meta-method initially computes a very private hypothesis, and then gradually subtracts noise (making the computation less and less private) until a sufficient level of accuracy is achieved. One key technique that saves significant factors in privacy loss over naive search is the use of correlated noise generated by the method of \cite{KHP15}, which formalizes the conceptual idea of ``subtracting'' noise without incurring additional privacy overhead.
In order to select the most private of these queries that meets the accuracy requirement, we introduce a natural modification of the now-classic AboveThreshold algorithm~\cite{DRbook}, which iteratively checks a sequence of queries on a dataset and privately releases the index of the first to approximately exceed some fixed threshold. Its privacy cost increases only logarithmically with the number of queries.
We provide an analysis of AboveThreshold that holds even if the queries themselves are the result of differentially private computations, showing that if AboveThreshold terminates after $t$ queries, one only pays the privacy costs of AboveThreshold plus the privacy cost of revealing those first $t$ private queries.
When combined with the above-mentioned correlated noise technique of~\cite{KHP15}, this gives an algorithm whose privacy loss is \emph{equal} to that of the final hypothesis output -- the previous ones coming ``for free'' -- plus the privacy loss of AboveThreshold.
Because the privacy guarantees achieved by this approach are not fixed a priori, but rather are a function of the data, we introduce and apply a new, corresponding privacy notion, which we term {\em ex-post} privacy, and which is closely related to the recently introduced notion of ``privacy odometers'' \cite{RRUV16}.

In Section~\ref{sec:experiments}, we empirically evaluate our noise reduction meta-method, which applies to any ERM technique which can be described as a post-processing of the Laplace mechanism. This includes both direct applications of the Laplace mechanism, like \emph{output perturbation} \cite{CMS11}; and more sophisticated methods like {\em covariance perturbation} \cite{AB17}, which perturbs the covariance matrix of the data and then performs an optimization using the noisy data. Our experiments concentrate on $\ell_2$ regularized least-squares regression and $\ell_2$ regularized logistic regression, and we apply our noise reduction meta-method to both output perturbation and covariance perturbation.
Our empirical results show that the active, ex-post privacy approach massively outperforms inverting the theory curve, and also improves on a baseline  ``$\epsilon$-doubling'' approach.


\section{Privacy Background and Tools}\label{sec:privtools}

\subsection{Differential Privacy and Ex-Post Privacy}

Let $\cX$ denote the data domain.
We call two \emph{datasets} $D, D' \in \cX^*$  \emph{neighbors}
(written as $D \sim D'$) if $D$ can be derived from $D'$ by replacing
a single data point with some other element of $\cX$.

\begin{definition}[Differential Privacy \cite{DMNS06}]
Fix $\eps \geq 0$. A randomized algorithm $A:\cX^*\rightarrow \mathcal{O}$ is $\eps$-differentially private if for every pair of neighboring data sets $D \sim D' \in \cX^*$, and for every event $S \subseteq \mathcal{O}$:
$$\Pr[A(D) \in S] \leq \exp(\eps)\Pr[A(D') \in S].$$
We call $\exp(\eps)$ the \emph{privacy risk} factor.
\end{definition}

It is possible to design computations that do not satisfy the
differential privacy definition, but whose outputs are private to an
extent that can be quantified after the computation halts.  For
example, consider an experiment that repeatedly runs an
$\eps'$-differentially private algorithm, until a stopping condition
defined by the output of the algorithm itself is met.  This experiment
does not satisfy $\eps$-differential privacy for any fixed value of
$\eps$, since there is no fixed maximum number of rounds for which the
experiment will run (for a fixed number of rounds, a simple composition
theorem, \Cref{composition}, shows that the $\eps$-guarantees in a
sequence of computations ``add up.'')  However, if ex-post we see that
the experiment has stopped after $k$ rounds, the data can in some
sense be assured an ``ex-post privacy loss'' of only $k \eps'$.
Rogers et al.~\cite{RRUV16} initiated the study of {\em privacy
  odometers}, which formalize this idea. Their goal was to develop a theory of privacy composition when the data analyst can choose the privacy parameters of subsequent computations as a function of the outcomes of previous computations.

We apply a related idea here, for a different purpose.
Our goal is to design one-shot algorithms that always achieve a target accuracy but that may have variable privacy levels depending on their input. 
\begin{definition}
  Given a randomized algorithm $\A: \cX^*\to \mathcal{O}$, define the
  \emph{ex-post privacy loss}\footnote{If $\A$'s output is from a
    continuous distribution rather than discrete, we abuse notation
    and write $\Pr[\A(D) = o]$ to mean the probability density at
    output $o$.}  of $\A$ on outcome $o$ to be
  \[ \Loss{o} = \max_{D,D' : D \sim D'} \log\frac{\prob{\A(D) = o}}{\prob{\A(D') = o}}. \]
\end{definition}
We refer to $\exp\left(\Loss{o}\right)$ as the {\em ex-post privacy risk}
factor.

\begin{definition}[Ex-Post Differential Privacy]
  Let $\mathcal{E}: \mathcal{O} \to (\mathbb{R}_{\geq 0} \cup \{\infty\})$ be a function on
  the outcome space of algorithm $\A: \cX^*\to \mathcal{O}$. Given an outcome $o = A(D)$, we say that $\A$ satisfies
  $\mathcal{E}(o)$-\textit{ex-post} differential privacy if for all $o \in \mathcal{O}$, $\Loss{o} \leq \mathcal{E}(o)$.
\end{definition}

Note that if $\mathcal{E}(o) \leq \eps$ for all $o$, $\A$ is $\epsilon$-differentially private. Ex-post differential privacy has the same semantics as differential privacy, once the output of the mechanism is known: it bounds the log-likelihood ratio of the dataset being $D$ vs. $D'$, which controls how an adversary with an arbitrary prior on the two cases can update her posterior.


\subsection{Differential Privacy Tools}

Differentially private computations enjoy two nice properties:
\begin{theorem}[Post Processing \cite{DMNS06}]
Let $A:\cX^*\rightarrow \mathcal{O}$ be any $\eps$-differentially private algorithm, and let $f:\mathcal{O}\rightarrow \mathcal{O'}$ be any function. Then the algorithm $f \circ A: \cX^*\rightarrow \mathcal{O}'$ is also $\eps$-differentially private.
\end{theorem}
Post-processing implies that, for example, every \emph{decision} process based on the output of a differentially private algorithm is also differentially private.

\begin{theorem}[Composition \cite{DMNS06}]\label{composition}
Let $A_1:\cX^*\rightarrow \mathcal{O}$, $A_2:\cX^*\rightarrow \mathcal{O}'$ be algorithms that are $\eps_1$- and $\eps_2$-differentially private, respectively. Then the algorithm $A:\cX^*\rightarrow \mathcal{O}\times \mathcal{O'}$ defined as $A(x) = (A_1(x), A_2(x))$ is $(\eps_1+\eps_2)$-differentially private.
\end{theorem}
The composition theorem holds even if the composition is \emph{adaptive}----see \cite{DRV10} for details.


{\bf The Laplace mechanism.}
The most basic subroutine we will use is the \emph{Laplace mechanism}.
The Laplace Distribution centered at $0$ with scale $b$ is the
distribution with probability density function
$\Lap{z|b} = \frac{1}{2b}e^{-\frac{|z|}{b}}$.
We say $X \sim \Lap{b}$ when $X$ has Laplace distribution with scale
$b$. Let $f\colon \cX^* \rightarrow \RR^d$ be an arbitrary
$d$-dimensional function. The {\em $\ell_1$ sensitivity} of $f$ is defined
to be $\Delta_1(f) = \max_{D\sim D'} \|f(D) - f(D')\|_1$.
The {\em Laplace mechanism} with parameter $\eps$ simply adds noise drawn independently
from $\Lap{\frac{\Delta_1(f)}{\eps}}$ to each coordinate of $f(x)$.

\begin{theorem}[\citep{DMNS06}]
  The Laplace mechanism is $\eps$-differentially private.
\end{theorem}


{\bf Gradual private release.} Koufogiannis et al.~\cite{KHP15} study
how to gradually release private data using the Laplace mechanism with
an increasing sequence of $\eps$ values, with a privacy cost scaling
only with the privacy of the \emph{marginal} distribution on the least
private release, rather than the sum of the privacy costs of
independent releases.
For intuition, the algorithm can be pictured as a continuous random
walk starting at some private data $v$ with the property that the
marginal distribution at each point in time is Laplace centered at
$v$, with variance increasing over time.  Releasing the value of the
random walk at a fixed point in time gives a certain output
distribution, for example, $\hat{v}$, with a certain privacy guarantee
$\eps$.  To produce $\hat{v}'$ whose \emph{ex-ante} distribution has
higher variance (is more private), one can simply ``fast forward'' the
random walk from a starting point of $\hat{v}$ to reach $\hat{v}'$; to
produce a less private $\hat{v}'$, one can ``rewind.''  The total
privacy cost is $\max\{\eps,\eps'\}$ because, given the ``least
private'' point (say $\hat{v}$), all ``more private'' points can be
derived as post-processings given by taking a random walk of a certain
length starting at $\hat{v}$.  Note that were the Laplace random
variables used for each release independent, the composition theorem
would require \emph{summing} the $\epsilon$ values of all releases.

In our private algorithms, we will use their noise reduction mechanism
as a building block to generate a list of private hypotheses
$\theta^1, \ldots, \theta^T$ with gradually increasing $\eps$
values. Importantly, releasing any prefix
$(\theta^1, \ldots, \theta^t)$ only incurs the privacy loss in
$\theta^t$. More formally:

\begin{algorithm}[h]
  \caption{Noise Reduction \cite{KHP15}
    : $\NR(v, \Delta, \{\eps_t\})$}
  \label{alg:gradualnoise}
  \begin{algorithmic}
    \State \textbf{Input:} private vector $v$, sensitivity parameter $\Delta$, list $\eps_1 < \eps_2 < \dots < \eps_T$
    \State Set $\hat{v}_T := v + \Lap{\Delta/\eps_T}$  \Comment{drawn i.i.d. for each coordinate}
    \For{$t=T-1,T-2,\dots,1$}
      \State With probability $\left(\frac{\eps_t}{\eps_{t+1}}\right)^2$: set $\hat{v}_t := \hat{v}_{t+1}$
      \State Else: set $\hat{v}_t := \hat{v}_{t+1} + \Lap{\Delta/\eps_t}$  \Comment{drawn i.i.d. for each coordinate}
    \EndFor
    \State Return $\hat{v}_1,\dots,\hat{v}_T$
  \end{algorithmic}
\end{algorithm}

\begin{theorem}[\cite{KHP15}] \label{thm:khp-noisereduction}
  Let $f$ have $\ell_1$ sensitivity $\Delta$ and let $\hat{v}_1,\dots,\hat{v}_T$ be the output of Algorithm \ref{alg:gradualnoise} on $v = f(D)$, $\Delta$, and the increasing list $\eps_1,\dots,\eps_T$.
  Then for any $t$, the algorithm which outputs the prefix $(\hat{v}_1,\dots,\hat{v}_t)$ is $\eps_t$-differentially private.
\end{theorem}

\subsection{AboveThreshold with Private Queries}
Our high-level approach to our eventual ERM problem will be as
follows: Generate a sequence of hypotheses $\theta_1,\dots,\theta_T$,
each with increasing accuracy and decreasing privacy; then test their
accuracy levels sequentially, outputting the first one whose accuracy
is ``good enough.''  The classical AboveThreshold
algorithm~\citep{DRbook} takes in a dataset and a sequence of queries
and privately outputs the index of the first query to exceed a given
threshold (with some error due to noise).  We would like to use
AboveThreshold to perform these accuracy checks, but there is an
important obstacle: for us, the ``queries'' themselves depend on the
private data.\footnote{In fact, there are many applications beyond
  our own in which the sequence of queries input to AboveThreshold
  might be the result of some private prior computation on the data,
  and where we would like to release both the stopping index of
  AboveThreshold and the ``query object.''  (In our case, the query
  objects will be parameterized by learned hypotheses
  $\theta_1,\dots,\theta_T$.)} A standard composition analysis would
involve first privately publishing \emph{all} the queries, then
running AboveThreshold on these queries (which are now public).
Intuitively, though, it would be much better to generate and publish
the queries one at a time, until AboveThreshold halts, at which point
one would not publish any more queries.  The problem with analyzing
this approach is that, a-priori, we do not know when AboveThreshold
will terminate; to address this, we analyze the \emph{ex-post
  privacy} guarantee of the algorithm.\footnote{This result does not follow from a
  straightforward application of privacy odometers from \cite{RRUV16},
  because the privacy analysis of algorithms like the noise reduction
  technique is not compositional.}

\begin{algorithm}[h]
  \caption{InteractiveAboveThreshold: $\IAT(D, \eps, W, \Delta, M)$}
  \label{alg:interactive-at}
  \begin{algorithmic}
    \State \textbf{Input:} Dataset $D$, privacy loss $\eps$, threshold
    $W$, $\ell_1$ sensitivity $\Delta$, algorithm $M$ \State Let
    $\hat{W} = W + \Lap{\frac{2\Delta}{\eps}}$ \For{ each query
      $t=1,\dots,T$} \State Query $f_t \leftarrow M(D)_t$ 
    \If{$f_t(D) + \Lap{\frac{4\Delta}{\eps}} \geq \hat{W}$:} { Output ($t$, $f_t$);  \textbf{Halt.}}
      \EndIf
    \EndFor
    \State Output ($T$, $\perp$).
  \end{algorithmic}
\end{algorithm}

Let us say that an algorithm $M(D) = (f_1,\ldots,f_T)$ is \emph{($\eps_1,\ldots,\eps_T)$-prefix-private} if for each $t$, the function that runs $M(D)$ and outputs just the prefix $(f_1,\ldots,f_t)$ is $\eps_t$-differentially private.
\begin{lemma}\label{lem:expost}
Let $M:\cX^*\rightarrow (\cX^*\rightarrow \mathcal{O})^T$ be a $(\eps_1,\ldots,\eps_T)$-prefix private algorithm that returns $T$ queries, and let each query output by $M$ have $\ell_1$ sensitivity at most $\Delta$. Then Algorithm \ref{alg:interactive-at} run on $D$, $\eps_A$, $W$, $\Delta$, and $M$ is $\mathcal{E}$-ex-post differentially private for
  $\mathcal{E}((t, \cdot)) = \eps_A + \eps_t$ for any $t\in [T]$.
\end{lemma}
The proof, which is a variant on the proof of privacy for AboveThreshold~\cite{DRbook}, appears in the appendix, along with an accuracy theorem for $\IAT$.



\begin{remark}
  Throughout we study $\eps$-differential privacy, instead of the
  weaker $(\eps, \delta)$ (approximate) differential privacy. Part of
  the reason is that an analogue of \Cref{lem:expost} does not seem
  to hold for $(\epsilon,\delta)$-differentially private queries without
  further assumptions, as the necessity to union-bound over the
  $\delta$ ``failure probability'' that the privacy loss is bounded
  for each query can erase the ex-post gains. We leave obtaining
  similar results for approximate differential privacy as an open
  problem.
\end{remark}



\section{Noise-Reduction with Private ERM}\label{sec:erm}

In this section, we provide a general private ERM framework that
allows us to approach the best privacy guarantee achievable on the data given a target excess
risk goal. Throughout the section, we consider an input dataset $D$
that consists of $n$ row vectors
$X_1, X_2, \ldots , X_n \in \mathbb{R}^p$ and a column
$y\in \mathbb{R}^n$. We will assume that each $\|X_i\|_1\leq 1$ and
$|y_i|\leq 1$. Let $d_i = (X_i , y_i)\in \RR^{p+1}$ be the $i$-th data
record. Let $\ell$ be a loss function such that for any hypothesis
$\theta$ and any data point $(X_i, y_i)$ the loss is
$\ell(\theta, (X_i, y_i))$. Given an input dataset $D$ and a
regularization parameter $\lambda$, the goal is to minimize the
following regularized empirical loss function over some feasible set
$C$:
  \[
    L(\theta, D) = \frac{1}{n} \sum_{i=1}^n \ell(\theta, (X_i, y_i)) +
    \frac{\lambda}{2} \|\theta\|_2^2.
  \]
  Let $\theta^* = \argmin_{\theta \in C} \ell(\theta, D)$. Given a
  target accuracy parameter $\alpha$, we wish to privately compute a
  $\theta_p$ that satisfies
  $L(\theta_p, D) \leq L(\theta^*, D) + \alpha$, while achieving the
  best ex-post privacy guarantee. For simplicity, we will sometimes
  write $L(\theta)$ for $L( \theta, D)$.

  One simple baseline approach is a ``doubling method'': Start with
  a small $\eps$ value, run an $\eps$-differentially private algorithm
  to compute a hypothesis $\theta$ and use the Laplace mechanism to
  estimate the excess risk of $\theta$; if the excess risk is lower
  than the target, output $\theta$; otherwise double the value of
  $\eps$ and repeat the same process. 
  (See the appendix for details.) As a result, we pay for privacy loss
  for every hypothesis we compute and every excess risk we estimate.

  In comparison, our meta-method provides a more cost-effective way to select the
  privacy level. The algorithm takes a more refined set of privacy
  levels $\eps_1 < \ldots < \eps_T$ as input and generates a sequence
  of hypotheses $\theta^1, \ldots, \theta^T$ such that the generation
  of each $\theta^t$ is $\eps_t$-private. Then it releases the
  hypotheses $\theta^t$ in order, halting as soon as a released hypothesis meets
  the accuracy goal. Importantly, there are two key components that
  reduce the privacy loss in our method:
  \begin{enumerate}
  \item We use \Cref{alg:gradualnoise}, the ``noise reduction'' method
    of \cite{KHP15}, for generating the sequence of hypotheses: we
    first compute a very private and noisy $\theta^1$, and then obtain
    the subsequent hypotheses by gradually ``de-noising''
    $\theta^1$. As a result, any prefix $(\theta^1, \ldots, \theta^k)$
    incurs a privacy loss of only $\eps_k$ (as opposed to
    $(\eps_1 + \ldots +\eps_k)$ if the hypotheses were independent).
  \item When evaluating the excess risk of each hypothesis, we use \Cref{alg:interactive-at}, InteractiveAboveThreshold, to determine if its excess risk exceeds the target threshold.
    This incurs substantially less privacy loss than independently evaluating the excess risk of each hypothesis using the Laplace mechanism (and hence allows us to search a finer grid of values).
\end{enumerate}

For the rest of this section, we will instantiate our method concretely
for two ERM problems: ridge regression and logistic regression.  In
particular, our noise-reduction method is based on two private ERM
algorithms: the recently introduced covariance perturbation technique of \cite{AB17}, and output perturbation
\cite{CMS11}.

\subsection{Covariance Perturbation for Ridge Regression}\label{sec:dataperturb}

In ridge regression, we consider the squared loss function:
$\ell((X_i, y_i), \theta) = \frac{1}{2}(y_i - \langle \theta,
X_i\rangle)^2$, and hence empirical loss over the data set is defined
as
\[
  L(\theta, D) = \frac{1}{2n} \|y - X\theta\|^2_2 + \frac{\lambda
    \|\theta\|_2^2}{2},
\]
where $X$ denotes the $(n\times p)$ matrix with row vectors
$X_1, \ldots, X_n$ and $y= (y_1, \ldots, y_n)$. Since the optimal
solution for the unconstrained problem has $\ell_2$ norm no more than
$\sqrt{1/\lambda}$ (see the appendix for a proof), we will focus on
optimizing $\theta$ over the constrained set
$C = \{a \in \RR^p \mid \|a\|_2\leq \sqrt{1/\lambda}\}$, which will be
useful for bounding the $\ell_1$ sensitivity of the empirical loss.


Before we formally introduce the covariance perturbation algorithm due to~\cite{AB17}, observe
that the optimal solution $\theta^*$ can be computed as
\[
  \theta^* = \argmin_{\theta\in C} L(\theta, D) = \argmin_{\theta\in
    C} \frac{ (\theta^\intercal (X^\intercal X) \theta - 2 \langle
    X^\intercal y, \theta\rangle)}{2n} + \frac{\lambda
    \|\theta\|_2^2}{2}.
\]
In other words, $\theta^*$ only depends on the private data through
$X^\intercal y$ and $X^\intercal X$. To compute a private hypothesis,
the covariance perturbation method simply adds Laplace noise to each entry
of $X^\intercal y$ and $X^\intercal X$ (the covariance matrix), and solves the optimization
based on the noisy matrix and vector. The formal description of the
algorithm and its guarantee are in~\Cref{data-perb}. Our analysis
differs from the one in~\cite{AB17} in that their paper considers the ``local privacy'' setting, and also adds Gaussian noise whereas we use Laplace. The proof is
deferred to the appendix.

\begin{theorem}\label{data-perb}
  Fix any $\eps >0$. 
  For any input data set $D$, consider the mechanism
  $\cM$ that computes
  $$
  \theta_p = \argmin_{\theta\in C} \frac{1}{2n}\left( \theta^\intercal
    (X^\intercal X + B)\theta - 2 \langle X^\intercal y + b,
    \theta\rangle \right) + \frac{\lambda \|\theta\|_2^2}{2},
  $$
  where $B\in \RR^{p\times p}$ and $b\in \RR^{p\times 1}$ are random
  Laplace matrices such that each entry of $B$ and $b$ is drawn from
  $\Lap{4/\eps}$. Then $\cM$ satisfies $\eps$-differential privacy and
  the output $\theta_p$ satisfies
  \[
    \Ex{B, b}{L(\theta_p) - L(\theta^*)} \leq
    \frac{4\sqrt{2}(2\sqrt{p/\lambda} + p/\lambda)}{n \eps}.
  \]
\end{theorem}

In our algorithm $\Dnoise$, we will apply the noise reduction
method, \Cref{alg:gradualnoise}, to produce a sequence of noisy versions of the private
data $(X^\intercal X, X^\intercal y)$:
$(Z^1, z^1), \ldots, (Z^T, z^T)$, one for each privacy level.  Then
for each $(Z^t, z^t)$, we will compute the private hypothesis by
solving the noisy version of the optimization problem
in~\Cref{betastuff}. The full description of our algorithm $\Dnoise$
is in~\Cref{alg:dnoise}, and satisfies the following guarantee:

\begin{algorithm}[h]
  \caption{Covariance Perturbation with Noise-Reduction:
    $\Dnoise(D,  \{\eps_1, \ldots, \eps_T\}, \alpha, \gamma)$}
 \label{alg:dnoise}
  \begin{algorithmic}
    \State{\textbf{Input:} private data set $D = (X, y)$, accuracy
      parameter $\alpha$, 
      privacy levels $\eps_1 < \eps_2 < \ldots < \eps_T$, and failure
      probability $\gamma$}

    \State{Instantiate InteractiveAboveThreshold:
      $\cA = \IAT(D, \eps_0, - \alpha / 2, \Delta,
      \cdot)$ 
      with $\eps_0 = 16\Delta (\log(2T/\gamma))/ \alpha$ and
      $\Delta = (\sqrt{1/\lambda} + 1)^2/(n)$ }

    \State{Let $C = \{a\in \RR^p \mid \|a\|_2\leq \sqrt{1/\lambda}\}$
      and $\theta^* = \argmin_{\theta\in C} L(\theta)$}


    \State{Compute noisy data:
      $$\{Z^t\} = \NR((X^\intercal X), 2, \{\eps_1/2, \ldots,
      \eps_T/2\}), \qquad \{z^t\} = \NR((X^\intercal Y), 2, \{\eps_1/2,
      \ldots, \eps_T/2\})$$ } \For{$t = 1, \ldots , T$:}

    \State{ \begin{equation}\theta^t = \argmin_{\theta\in C} \frac{1}{2n}\left(
          \theta^\intercal Z^t\theta - 2 \langle
          z^t, \theta\rangle \right) + \frac{\lambda
          \|\theta\|_2^2}{2} \label{betastuff}
      \end{equation}
}

    \State{Let $f^t(D) = L(\theta^*, D) - L(\theta^t, D)$; Query $\cA$ with query $f^t$ to check accuracy}

    \If{ $\cA$ returns $(t, f^t)$}{ \textbf{Output} $(t, \theta^t)$} \Comment{Accurate hypothesis found.}
\EndIf
      \EndFor

      \State{\textbf{Output:} $(\perp, \theta^*)$}
    \end{algorithmic}
  \end{algorithm}

\begin{theorem}\label{thm:mainc}
  The instantiation of
  $\Dnoise(D, \{\eps_1, \ldots, \eps_T\}, \alpha, \gamma)$ outputs a
  hypothesis $\theta_p$ that with probability $1 - \gamma$ satisfies
  $L(\theta_p) - L(\theta^*) \leq \alpha$.  Moreover, it is
  $\mathcal{E}$-ex-post differentially private, where the privacy loss
  function $\mathcal{E} \colon (([T] \cup \{\perp\}) \times \RR^p)\to (\mathbb{R}_{\geq 0} \cup \{\infty\})$ is
  defined as $\mathcal{E}((k , \cdot)) = \eps_0 + \eps_k$ for any
  $k\neq \perp$,  $\mathcal{E}((\perp, \cdot)) = \infty$, and
  $\eps_0 = \frac{16(\sqrt{1/\lambda} + 1)^2 \log(2T/\gamma)}{n\alpha}$
  is the privacy loss incurred by $\IAT$.
\end{theorem}

\subsection{Output Perturbation for Logistic Regression}
Next, we show how to combine the output perturbation method with
noise reduction for the ridge regression problem.\footnote{We study
  the ridge regression problem for concreteness. Our method works for
  any ERM problem with strongly convex loss functions.}  In this
setting, the input data consists of $n$ labeled examples
$(X_1 , y_1) ,\ldots , (X_n, y_n)$, such that for each $i$,
$X_i \in \RR^p$, $\|X_i\|_1 \leq 1$, and $y_i\in \{-1, 1\}$. The goal
is to train a linear classifier given by a weight vector $\theta$ for
the examples from the two classes. We consider the logistic loss
function:
$\ell(\theta, (X_i, y_i)) = \log(1 + \exp(-y_i\theta^\intercal X_i))$,
and the empirical loss is
\[
  L(\theta, D) = \frac{1}{n} \sum_{i=1}^n \log(1 + \exp(- y_i
  \theta^\intercal X_i)) + \frac{\lambda\|\theta\|_2^2}{2}.
\]


The output perturbation method is straightforward: we simply add
Laplace noise to perturb each coordinate of the optimal solution
$\theta^*$. The following is the formal guarantee of output
perturbation. Our analysis deviates slightly from the one
in~\cite{CMS11} since we are adding Laplace noise (see the appendix).

\begin{theorem}\label{logistman}
  Fix any $\eps > 0$. Let $r = \frac{2\sqrt{p}}{n\lambda\eps}$. For
  any input dataset $D$, consider the mechanism that first computes
  $\theta^* = \argmin_{\theta\in \RR^p} L(\theta)$, then outputs
  $\theta_p = \theta^* + b$, where $b$ is a random vector with its
  entries drawn i.i.d. from $\Lap{r}$. Then $\cM$ satisfies
  $\eps$-differential privacy, and $\theta_p$ has excess risk
    $\Ex{b}{L(\theta_p) - L(\theta^*)} \leq
    \frac{2\sqrt{2}p}{n\lambda\eps} + \frac{4p^2}{n^2 \lambda \eps^2}.$
\end{theorem}

Given the output perturbation method, we can simply apply the noise
reduction method $\NR$ to the optimal hypothesis $\theta^*$ to generate a sequence of noisy hypotheses. We
will again use InteractiveAboveThreshold to check the excess risk of the
hypotheses. The full algorithm $\Onoise$ follows the same structure in
\Cref{alg:dnoise}, and we defer the formal description to the
appendix.

\begin{theorem}\label{thm:maino}
  The instantiation of
  $\Onoise(D, \eps_0, \{\eps_1, \ldots, \eps_T\}, \alpha, \gamma)$ is
  $\mathcal{E}$-ex-post differentially private and outputs a hypothesis
  $\theta_p$ that with probability $1 - \gamma$ satisfies
  $L(\theta_p) - L(\theta^*) \leq \alpha$, where the privacy loss
  function $\mathcal{E} \colon (([T] \cup \{\perp\}) \times \RR^p)\to (\mathbb{R}_{\geq 0} \cup \{\infty\})$ is
  defined as $\mathcal{E}((k , \cdot )) = \eps_0 + \eps_k$ for any
  $k\neq \perp$,  $\mathcal{E}((\perp, \cdot)) = \infty$, and
  $\eps_0 \leq \frac{32\log(2T/\gamma)\sqrt{2\log 2/\lambda}}{n\alpha}$
  is the privacy loss incurred by $\IAT$.
\end{theorem}

\begin{proof}[Proof sketch of~\Cref{thm:mainc,thm:maino}]
  The accuracy guarantees for both algorithms follow from an accuracy
  guarantee of the $\IAT$ algorithm (a variant on the standard AboveThreshold bound) and the fact that we output
  $\theta^*$ if $\IAT$ identifies no accurate hypothesis. For the
  privacy guarantee, first note that any prefix of the noisy
  hypotheses $\theta^1, \ldots , \theta^t$ satisfies
  $\eps_t$-differential privacy because of our instantiation of the
  Laplace mechanism (see the appendix for the $\ell_1$ sensitivity
  analysis) and noise-reduction method $\NR$. Then the ex-post privacy
  guarantee directly follows \Cref{lem:expost}.
\end{proof}


\section{Experiments}\label{sec:experiments}
To evaluate the methods described above, we conducted empirical evaluations in two settings.
We used ridge regression to predict (log) popularity of posts on Twitter in the dataset of \cite{BuzzDataset}, with $p=77$ features and subsampled to $n=$100,000 data points.
Logistic regression was applied to classifying network events as innocent or malicious in the KDD-99 Cup dataset~\cite{KDDCupDataset}, with 38 features and subsampled to 100,000 points.
Details of parameters and methods appear in the appendix.

In each case, we tested the algorithm's average ex-post privacy loss for a range of input accuracy goals $\alpha$, fixing a modest failure probability $\gamma=0.1$ (and we observed that excess risks were concentrated well below $\alpha/2$, suggesting a pessimistic analysis).
The results show our meta-method gives a large improvement over the ``theory'' approach of simply inverting utility theorems for private ERM algorithms.
(In fact, the utility theorem for the popular private stochastic gradient descent algorithm does not even give meaningful guarantees for the ranges of parameters tested; one would need an order of magnitude more data points, and even then the privacy losses are enormous, perhaps due to loose constants in the analysis.)

To gauge the more modest improvement over $\DM$, note that the variation in the privacy risk factor $e^{\eps}$ can still be very large; for instance, in the ridge regression setting of $\alpha=0.05$, Noise Reduction has $e^{\eps} \approx 10.0$ while $\DM$ has $e^{\eps} \approx 495$; at $\alpha=0.075$, the privacy risk factors are $4.65$ and $56.6$ respectively.

\begin{figure}
  \captionsetup[subfigure]{justification=centering}
  \begin{subfigure}{0.48\linewidth}
    \includegraphics[width=\linewidth]{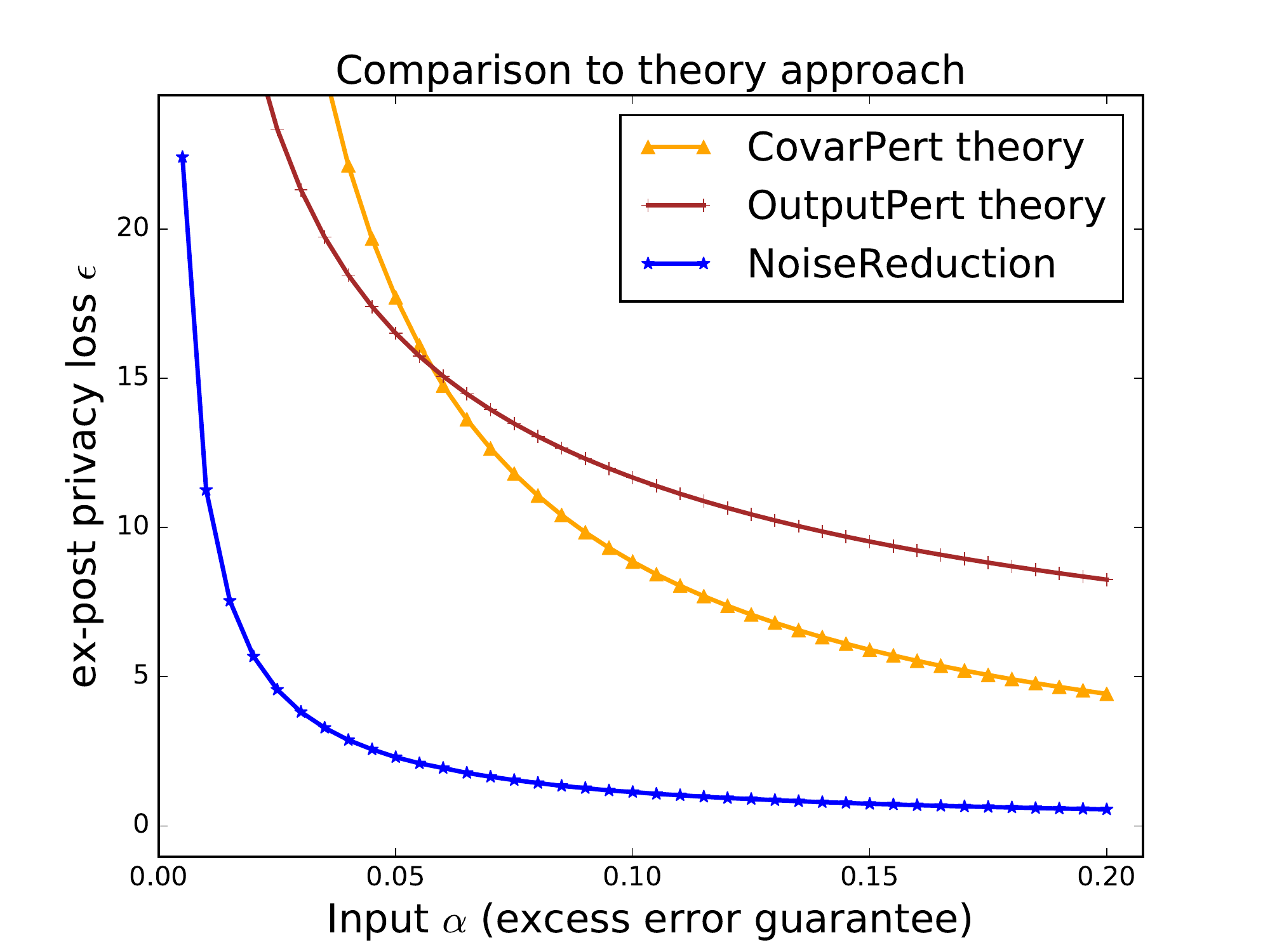}
    \caption{\textbf{Linear (ridge) regression,\\ vs theory approach.}}
    \label{subfig:privacy-twitter-ridge-theory}
  \end{subfigure}
  \begin{subfigure}{0.48\linewidth}
    \includegraphics[width=\linewidth]{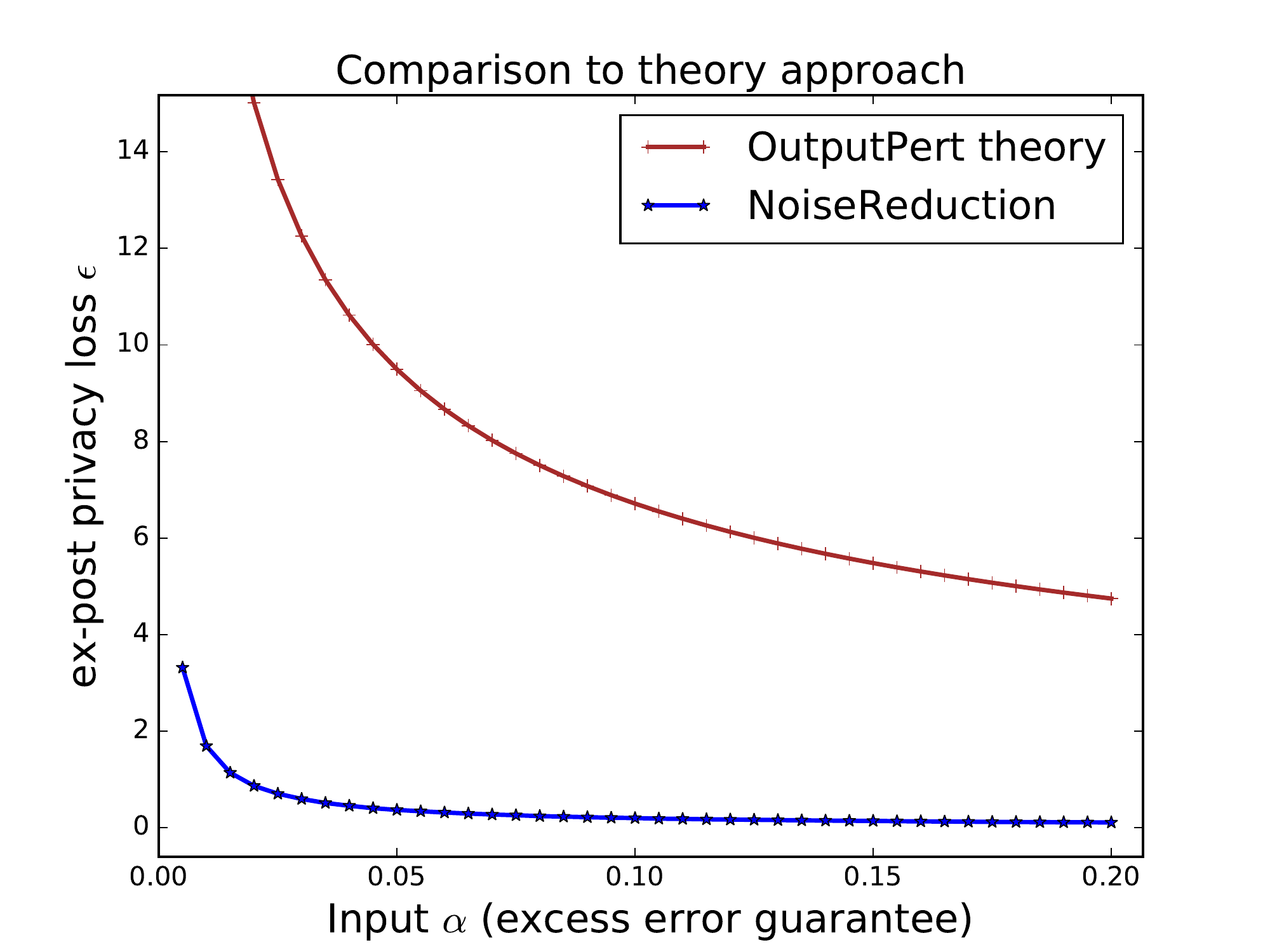}
    \caption{\textbf{Regularized logistic regression,\\ vs theory approach.}}
    \label{subfig:privacy-kdd-logist-theory}
  \end{subfigure}

  \begin{subfigure}{0.48\linewidth}
    \includegraphics[width=\linewidth]{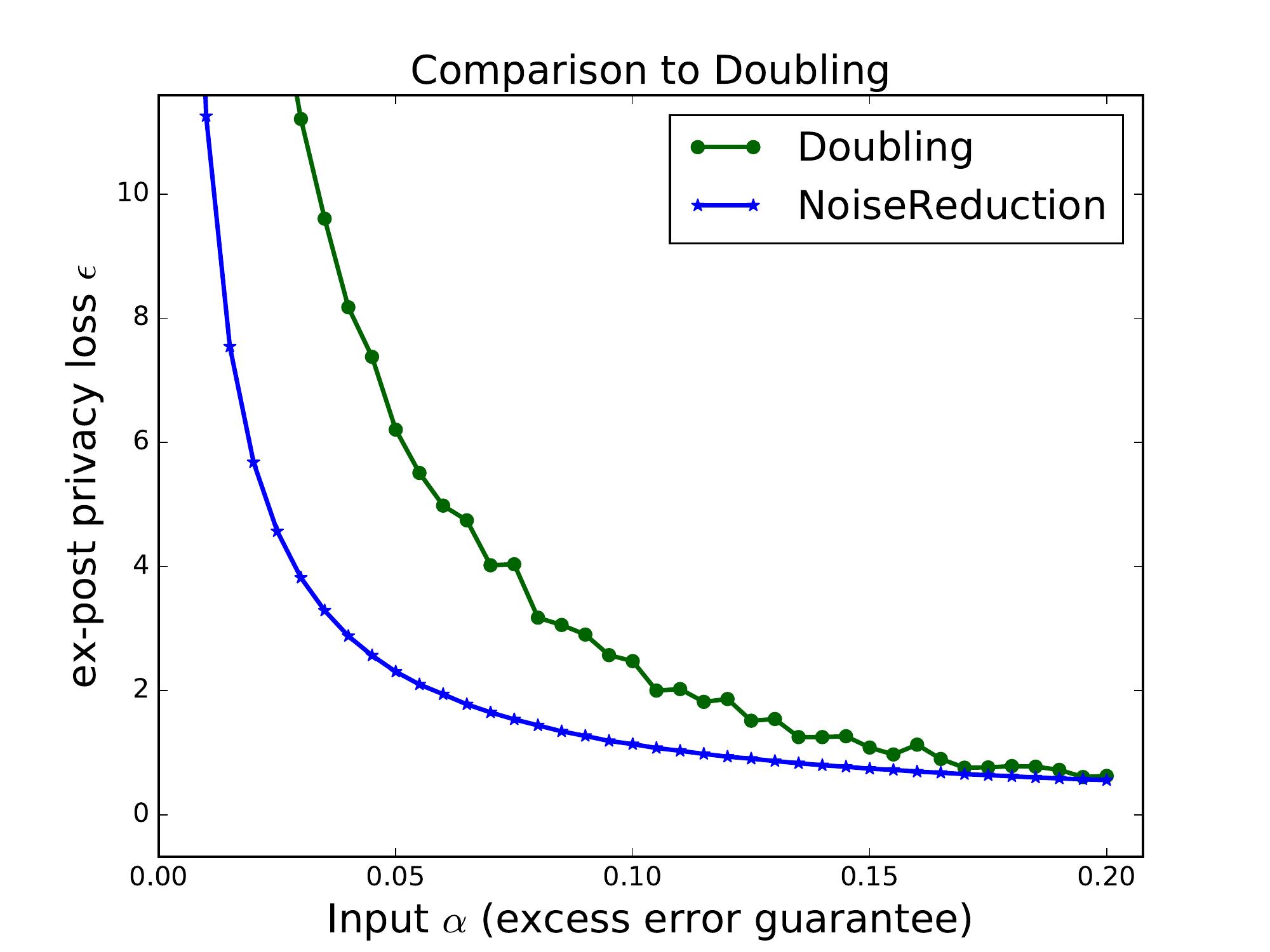}
    \caption{\textbf{Linear (ridge) regression,\\ vs $\DM$.}}
    \label{subfig:privacy-twitter-ridge-naive}
  \end{subfigure}
  \begin{subfigure}{0.48\linewidth}
    \includegraphics[width=\linewidth]{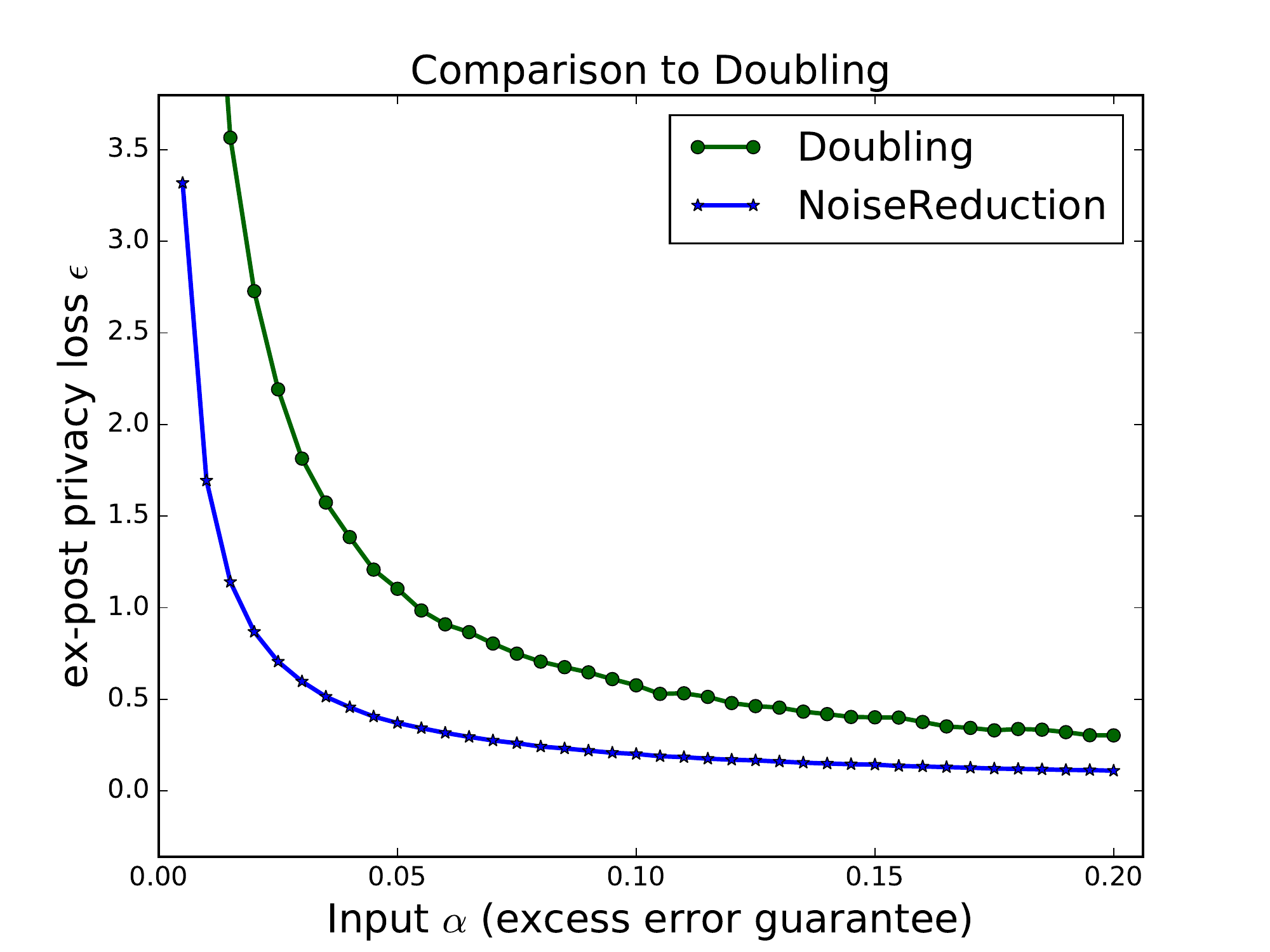}
    \caption{\textbf{Regularized logistic regression,\\ vs $\DM$.}}
    \label{subfig:privacy-kdd-logist-naive}
  \end{subfigure}
  \caption{\textbf{Ex-post privacy loss.} (\ref{subfig:privacy-twitter-ridge-theory}) and (\ref{subfig:privacy-twitter-ridge-naive}), left, represent  ridge regression on the Twitter dataset, where Noise Reduction and $\DM$ both use Covariance Perturbation. (\ref{subfig:privacy-kdd-logist-theory}) and (\ref{subfig:privacy-kdd-logist-naive}), right, represent logistic regression on the KDD-99 Cup dataset, where both Noise Reduction and $\DM$ use Output Perturbation. The top plots compare Noise Reduction to the ``theory approach'': running the algorithm once using the value of $\epsilon$ that guarantees the desired expected error via a utility theorem. The bottom compares to the $\DM$ baseline. Note the top plots are generous to the theory approach: the theory curves promise only expected error, whereas Noise Reduction promises a high probability guarantee. Each point is an average of 80 trials (Twitter dataset) or 40 trials (KDD-99 dataset).}
  \label{fig:privacy-plots}
\end{figure}

Interestingly, for our meta-method, the contribution to privacy loss from ``testing'' hypotheses (the InteractiveAboveThreshold technique) was significantly larger than that from ``generating'' them (NoiseReduction).
One place where the InteractiveAboveThreshold analysis is loose is in using a theoretical bound on the maximum norm of any hypothesis to compute the sensitivity of queries.
The actual norms of hypotheses tested was significantly lower which, if taken as guidance to the practitioner in advance, would drastically improve the privacy guarantee of both adaptive methods.



\section*{Acknowledgements}
This work was supported in part by NSF grants CNS-1253345, CNS-1513694, CNS-1254169 and CNS-1518941, US-Israel Binational Science Foundation grant 2012348, Israeli Science Foundation (ISF) grant  \#1044/16, a subcontract on the DARPA Brandeis Project, the Warren Center for Data and Network Sciences, and the HUJI Cyber Security Research Center in conjunction with the Israel National Cyber Bureau in the Prime Minister's Office.

\newpage
\bibliographystyle{plain}

\bibliography{refs}

\begin{thebibliography}{10}

\bibitem{BuzzDataset}
The AMA~Team at~Laboratoire d'Informatique~de Grenoble.
\newblock Buzz prediction in online social media, 2017.

\bibitem{BST14}
Raef Bassily, Adam~D. Smith, and Abhradeep Thakurta.
\newblock Private empirical risk minimization, revisited.
\newblock {\em CoRR}, abs/1405.7085, 2014.

\bibitem{DBLP:conf/nips/ChaudhuriM08}
Kamalika Chaudhuri and Claire Monteleoni.
\newblock Privacy-preserving logistic regression.
\newblock In {\em Advances in Neural Information Processing Systems 21,
  Proceedings of the Twenty-Second Annual Conference on Neural Information
  Processing Systems, Vancouver, British Columbia, Canada, December 8-11,
  2008}, pages 289--296, 2008.

\bibitem{CMS11}
Kamalika Chaudhuri, Claire Monteleoni, and Anand~D. Sarwate.
\newblock Differentially private empirical risk minimization.
\newblock {\em Journal of Machine Learning Research}, 12:1069--1109, 2011.

\bibitem{DBLP:conf/allerton/DuchiJW13}
John~C. Duchi, Michael~I. Jordan, and Martin~J. Wainwright.
\newblock Local privacy and statistical minimax rates.
\newblock In {\em 51st Annual Allerton Conference on Communication, Control,
  and Computing, Allerton 2013, Allerton Park {\&} Retreat Center, Monticello,
  IL, USA, October 2-4, 2013}, page 1592, 2013.

\bibitem{DMNS06}
Cynthia Dwork, Frank McSherry, Kobbi Nissim, and Adam Smith.
\newblock Calibrating noise to sensitivity in private data analysis.
\newblock In {\em Theory of Cryptography Conference}, pages 265--284. Springer,
  2006.

\bibitem{DRbook}
Cynthia Dwork and Aaron Roth.
\newblock The algorithmic foundations of differential privacy.
\newblock {\em Foundations and Trends{\textregistered} in Theoretical Computer
  Science}, 9(3--4):211--407, 2014.

\bibitem{DRV10}
Cynthia Dwork, Guy~N Rothblum, and Salil Vadhan.
\newblock Boosting and differential privacy.
\newblock In {\em Foundations of Computer Science (FOCS), 2010 51st Annual IEEE
  Symposium on}, pages 51--60. IEEE, 2010.

\bibitem{RAPPOR}
Giulia Fanti, Vasyl Pihur, and Úlfar Erlingsson.
\newblock Building a rappor with the unknown: Privacy-preserving learning of
  associations and data dictionaries.
\newblock {\em Proceedings on Privacy Enhancing Technologies (PoPETS)}, issue
  3, 2016, 2016.

\bibitem{apple}
Andy Greenberg.
\newblock Apple's 'differential privacy' is about collecting your data---but
  not your data.
\newblock {\em Wired Magazine}, 2016.

\bibitem{DBLP:journals/jmlr/JainKT12}
Prateek Jain, Pravesh Kothari, and Abhradeep Thakurta.
\newblock Differentially private online learning.
\newblock In {\em {COLT} 2012 - The 25th Annual Conference on Learning Theory,
  June 25-27, 2012, Edinburgh, Scotland}, pages 24.1--24.34, 2012.

\bibitem{KDDCupDataset}
KDD'99.
\newblock Kdd cup 1999 data, 1999.

\bibitem{DBLP:journals/jmlr/KiferST12}
Daniel Kifer, Adam~D. Smith, and Abhradeep Thakurta.
\newblock Private convex optimization for empirical risk minimization with
  applications to high-dimensional regression.
\newblock In {\em {COLT} 2012 - The 25th Annual Conference on Learning Theory,
  June 25-27, 2012, Edinburgh, Scotland}, pages 25.1--25.40, 2012.

\bibitem{KHP15}
Fragkiskos Koufogiannis, Shuo Han, and George~J. Pappas.
\newblock Gradual release of sensitive data under differential privacy.
\newblock {\em Journal of Privacy and Confidentiality}, 7, 2017.

\bibitem{MT07}
Frank McSherry and Kunal Talwar.
\newblock Mechanism design via differential privacy.
\newblock In {\em Foundations of Computer Science, 2007. FOCS'07. 48th Annual
  IEEE Symposium on}, pages 94--103. IEEE, 2007.

\bibitem{RRUV16}
Ryan~M Rogers, Aaron Roth, Jonathan Ullman, and Salil Vadhan.
\newblock Privacy odometers and filters: Pay-as-you-go composition.
\newblock In D.~D. Lee, M.~Sugiyama, U.~V. Luxburg, I.~Guyon, and R.~Garnett,
  editors, {\em Advances in Neural Information Processing Systems 29}, pages
  1921--1929. Curran Associates, Inc., 2016.

\bibitem{DBLP:journals/corr/abs-0911-5708}
Benjamin I.~P. Rubinstein, Peter~L. Bartlett, Ling Huang, and Nina Taft.
\newblock Learning in a large function space: Privacy-preserving mechanisms for
  {SVM} learning.
\newblock {\em CoRR}, abs/0911.5708, 2009.

\bibitem{AB17}
Adam Smith, Jalaj Upadhyay, and Abhradeep Thakurta.
\newblock Is interaction necessary for distributed private learning?
\newblock {\em IEEE Symposium on Security and Privacy}, 2017.

\bibitem{DBLP:conf/globalsip/SongCS13}
Shuang Song, Kamalika Chaudhuri, and Anand~D. Sarwate.
\newblock Stochastic gradient descent with differentially private updates.
\newblock In {\em {IEEE} Global Conference on Signal and Information
  Processing, GlobalSIP 2013, Austin, TX, USA, December 3-5, 2013}, pages
  245--248, 2013.

\bibitem{DBLP:conf/nips/WilliamsM10}
Oliver Williams and Frank McSherry.
\newblock Probabilistic inference and differential privacy.
\newblock In {\em Advances in Neural Information Processing Systems 23: 24th
  Annual Conference on Neural Information Processing Systems 2010. Proceedings
  of a meeting held 6-9 December 2010, Vancouver, British Columbia, Canada.},
  pages 2451--2459, 2010.

\end{thebibliography}

\newpage
\appendix

\section{Missing Details and Proofs}
\subsection{AboveThreshold}

\begin{proof}[Proof of \Cref{lem:expost}]
Let $D, D'$ be  neighboring databases.
We will instead analyze the algorithm that outputs the entire prefix $f_1,\dots,f_t$ when stopping at time $t$.
Because \IAT\ is a post-processing of this algorithm, and privacy can only be improved under post-processing, this suffices to prove the theorem.
We wish to show for all outcomes $o = (t, f_1,\dots,f_t)$:
\[
	\prob{\IAT(D) = (t, f_1, f_2, \ldots, f_t)} \leq e^{\eps_A + \eps_t}\prob{\IAT(D') = (t, f_1, f_2, \ldots, f_t)}.
\]

We have directly from the privacy guarantee of InteractiveAboveThreshold that for every \emph{fixed} sequence of queries $f_1,\ldots,f_t$:
\begin{equation}\label{cond}\prob{\IAT(D) = t \mid f_1, \ldots, f_t}\leq e^{\eps_A} \prob{\IAT(D') = t \mid f_1, \ldots, f_t}\end{equation}
because the guarantee of InteractiveAboveThreshold is quantified over all data-independent sequences of queries $f_1,\ldots,f_T$, and by definition of the algorithm, the probability of stopping at time $t$ is independent of the identity of any query $f_t'$ for $t' > t$.

Now we can write:
 $$\prob{\IAT(D) = t, f_1, \ldots f_t} = \prob{\IAT(D) = t \mid f_1, \ldots f_t}\prob{M(D) = f_1, \ldots f_t}.$$

 By assumption, $M$ is prefix-private, in particular, for fixed $t$ and any $f_1,\ldots,f_t$:
  $$\prob{M(D) = f_1, \ldots f_t} \leq e^{\eps_t}\prob{M(D') = f_1, \ldots f_t}$$
 Thus,
\begin{align*}
  \frac{\prob{\IAT(D) = t, f_1, \ldots f_t} }{\prob{\IAT(D') = t, f_1, \ldots f_t}}
  &= \frac{\prob{\IAT(D) = t \mid f_1, \ldots f_t}}{\prob{\IAT(D') = t | f_1, \ldots, f_t}}\frac{\prob{M(D) = f_1, \ldots f_t}}{\prob{M(D') = f_1, \ldots f_t}} \\
  &\leq e^{\eps_A}\cdot e^{\eps_t} = e^{\eps_A + \eps_t},
\end{align*}
 as desired.
\end{proof}

We also include the following utility theorem. We say that an
instantiation of InteractiveAboveThreshold is $(\alpha, \beta)$ accurate with
respect to a threshold $W$ and stream of queries $f_1, \ldots f_T$ if
except with probability at most $\gamma$, the algorithm outputs a
query $f_t$ only if $f_t(D) \geq W - \alpha$.

\begin{theorem}
  For any sequence of 1-sensitive queries $f_1, \ldots , f_T$ such
  InteractiveAboveThreshold is $(\alpha, \beta)$-accurate for
 \[
   \alpha = \frac{8 \Delta(\log(T) + \log(2/\gamma))}{\eps}.
 \]
\end{theorem}

\subsection{Doubling Method}
We now formally describe the $\DM$ discussed in Section~\ref{sec:intro} and Section~\ref{sec:erm}, and give a formal ex-post privacy analysis. Let $\theta^* = \argmin_{\theta \in \mathbb{R}^p}L(\theta)$.
$\DM$ accepts a list of privacy levels $\eps_1 < \eps_2 < \ldots < \eps_T,$ where $\eps_i = 2\eps_{i-1}$. We show in \Cref{claim:doubling-2-opt} that $2$ is 
the optimal factor to scale $\eps$ by.  It also takes in a failure probability $\gamma$, 
and a black-box private ERM mechanism $M$ that has the following guarantee: Fixing a dataset $D$, $M$ takes as input $D$ and a privacy level $\eps_i$, 
and generates an $\eps_i$-differentially private hypothesis $\theta_i$, such that the query $f^{i}(D) = L(D, \theta*)- L(D, \theta_i)$ has $\ell_1$ sensitivity at most  $\Delta$.

\begin{algorithm}[h]
  \caption{Doubling Method: $\DM(D,\{ \eps_1, \ldots, \eps_T\}, M, \alpha, \gamma)$}
  \label{alg:doubling}
  \begin{algorithmic}
    \State \textbf{Input:} private dataset $D$, accuracy $\alpha$, failure probability $\gamma$, mechanism $M$
     \State \For{each  $t=1,\dots,T$}
     \State Generate  $\theta_t \leftarrow M(D)_t$ 
     \State Let $f^t(D) = L(D, \theta^*)- L(D, \theta_t)$
     \State Generate $w_t \sim \Lap{\frac{\alpha}{2\log(\frac{T}{\gamma})}}$
    \If{$f^t(D) + w_t \geq -\alpha/2$:} { Output ($t$, $f^t$);  \textbf{Halt.}}
    \EndIf
    \EndFor
    \State Output $T+1, \theta^*$.
  \end{algorithmic}
\end{algorithm}
\begin{theorem}\label{thm:double}
For $k \leq T$, define the privacy loss function $\mathcal{E}(k, \theta_k) = \frac{2k\Delta\log(T/\gamma)}{\alpha} + (2^{k}-1)\eps_1, and \mathcal{E}(T+1,\theta^{*}) = \infty$. Then $\DM$ is $\mathcal{E}$-ex-post differentially private, and is $1-\gamma$ accurate.
\end{theorem}
\begin{proof}
Since if the algorithm reaches step $T+1$ it outputs the true minimizer which has error $0 < \alpha$, it could only 
fail to output a hypothesis with error less than $\alpha$ if it stops at $i \leq T$. $\DM$ only stops early if the noisy query is 
greater than $-\alpha/2$; or $f^i(D) + w_i \geq -\alpha/2$. But $f^i(D) \leq -\alpha$, which forces $w_i \geq \alpha/2$. 
By properties of the Laplace distribution, $\prob{w_i \geq \alpha/2} = \frac{1}{2}\text{exp}(\frac{-\alpha}{2}\frac{2\log(\frac{T}{2\gamma})}{\alpha}) = \gamma/T$. Hence by 
union bound over $T$ the total failure probability is at most $\gamma$. 

By the assumption, generating the $k^{th}$ private hypothesis incurs privacy loss $\epsilon_1*2^{k-1}$. By the Laplace mechanism, 
evaluating the error of the sensitivity $\Delta$ query $f^{i}$ is $\frac{2\Delta\log(T/\gamma)}{\alpha}$-differentially private. Theorem $3.6$ in \cite{RRUV16}
then says that the ex-post privacy loss of outputting $k \leq T$ is $\sum_{i = 1}^{k}[\epsilon_1*2^{k-1} + \frac{2\Delta\log(T/\gamma)}{\alpha}] = \frac{2k\Delta\log(T/\gamma)}{\alpha} + (2^{k}-1)\eps_1$,
as desired. 
\end{proof}
\begin{remark}
In practice, the private empirical risk minimization mechanism $M$ may not always output 
a hypothesis that leads to queries with uniformly bounded $\ell_1$ sensitivity. In this case, a projection 
that scales down, the hypothesis norm can be applied prior to evaluating the private query error. For a discussion of
scaling the norm down refer to the experiments section of the appendix.  
\end{remark}
\sn{where do we discuss this choice?}
\subsection{Ridge Regression}
In this subsection, we let
$\ell(\theta , (X_i, y_i)) = \frac{1}{2}(y_i - \langle \theta,
X_i\rangle)^2$, and the empirical loss over the data set is defined as
\[
  L(D, \theta) = \frac{1}{2n} \|y - X\theta\|^2_2 + \frac{\lambda
    \|\theta\|_2^2}{2},
\]
where $X$ denotes the $(n\times p)$ matrix with row vectors
$X_1, \ldots, X_n$ and $y= (y_1, \ldots, y_n)$. We assume that for
each $i$, $\|X_i\|_1\leq 1$ and $|y_i|\leq 1$. For simplicity, we will
sometimes write $L(\theta)$ for $L(D, \theta)$.

First, we show that the unconstrained optimal solution in ridge
regression has bounded norm.

\begin{lemma}\label{lem:norm_bound}
  Let $\theta^* = \argmin_{\theta\in \RR^d} L(\theta)$. Then
  $||\theta^*||_{2} \leq \frac{1}{\sqrt{\lambda}}$.
\end{lemma}
\begin{proof}
  For any $\theta\in \RR^p, L(\theta^*) \leq L(\theta)$. In particular
  for $\theta = \mathbf{0}$,
  $$ L(\theta^*) \leq L(\mathbf{0}) = \sum_{i =
    1}^{n}\frac{1}{2n}\ell((X_i, y_i), 0) \leq \frac{1}{2}.$$ Note
  that for any $\theta, \ell((X_i,y_i), \theta) \geq 0$, so this means
  $L(\theta^*) \geq \frac{\lambda}{2}||\theta^*||_2^2$, which forces
  $\frac{\lambda}{2}||\theta^*||_2^2 \leq \frac{1}{2}$, and so
  $||\theta^*||_{2} \leq \frac{1}{\sqrt{\lambda}}$ as desired.
\end{proof}

The following claim provides a bound on the sensitivity for the excess
risk, which are the queries we send to InteractiveAboveThreshold.

\begin{claim}\label{3sen-ridge}
  Let $C$ be a bounded convex set in $\RR^p$ with $\|C\|_2\leq M$. Let
  $D$ and $D'$ be a pair of adjacent datasets, and let
  $\theta^* = \argmin_{\theta\in C} L(\theta, D)$ and
  $\theta^\bl =\argmin_{\theta\in C} L(\theta, D')$.
Then for any $\theta\in C$,
\[
  |(L(\theta, D) - L(\theta^*, D) )- (L(\theta, D') - L(\theta^\bl, D'))| \leq \frac{(M+1)^2}{n}.
\]
\end{claim}

The following lemma provides a bound on the $\ell_1$ sensitivity for
the matrix $X^\intercal X$ and vector $X^\intercal y$.

\begin{lemma}\label{4sen}
  Fix any $i\in [n]$.  Let $X$ and $Z$ be two $n\times p$ matrices
  such that for all rows $j\neq i$, $X_j = Z_j$. Let $y,y'\in \RR^n$
  such that $y_j = y_j'$ for all $j\neq i$.  Then
  \[
    \| X^\intercal X - Z^\intercal Z \|_1 \leq 2 \quad \mbox{ and } \quad
    \| X^\intercal y - Z^\intercal y'\|_1 \leq 2,
  \]
  as long as $\|X_i\| , \|Z_i\|, |y_i|, |y_i'|\leq 1$.
\end{lemma}

\begin{proof}
We can write
\begin{align*}
  \| X^\intercal X - Z^\intercal Z\|_1 &= \|\sum_j \left(X^\intercal_j X_j -  Z^\intercal_j Z_j\right)\|_1\\
                                       &= \|X^\intercal_i X_i - Z^\intercal_i Z\|_1\\
                                       &\leq \|X^\intercal_i X_i\|_1 + \|Z^\intercal_i Z_i\|_1\\
                                       &= \|X_i\|_1^2 + \|Z_i\|_1^2 \leq 2.
\end{align*}
Similarly,
\begin{align*}
  \| X^\intercal y - Z^\intercal y'\|_1 &= \|\sum_j \left(y_j X_j  -  y'_j Z_j \right)\|_1\\
                                       &= \|y_i X_i - y'_i Z_i\|_1\\
                                       &= \|y_i X_i\|_1 + \|y'_i Z_i\|_1\\
                                       &= \|X_i\|_1 + \|Z_i\|_1\leq 2.
\end{align*}
This completes the proof.
\end{proof}

Before we proceed to give a formal proof for \Cref{data-perb}, we will
also give the following basic fact about Laplace random vectors.

\begin{claim}\label{jensen}
  Let $\nu = (\nu_1, \ldots, \nu_k)\in \RR^k$ such that each $\nu_i$
  is an independent random variable drawn from the Laplace
  distribution $\Lap{r}$. Then $\Ex{}{\|\nu\|_2} \leq \sqrt{2k}r$.
\end{claim}

\begin{proof}
By Jensen's inequality, 
\[
  \Ex{}{\|\nu\|_2} =  \Ex{}{\sqrt{\sum_i \nu_i^2}}  \leq \sqrt{ \Ex{}{\sum_i \nu_i^2}}.
\]
Note that by linearity of expectation and the variance of the Laplace
distribution
\[
  \Ex{}{\sum_i \nu_i^2} = \sum_i \Ex{}{\nu_i^2} = \sum_i 2r^2 = 2k r^2.
\]
Therefore, we have $\Ex{}{\|\nu\|_2} \leq \sqrt{2k} r$.
\end{proof}

\begin{proof}[Proof of~\Cref{data-perb}]
  
  In the algorithm, we compute $Z = X^\intercal X + B$ and
  $z = X^\intercal y + b$, where the entries of $B$ and $b$ are drawn
  i.i.d. from $\Lap{4/\eps}$. Note that the output $\theta_p$ is
  simply a post-processing of the noisy matrix $Z$ and vector
  $z$. Furthermore, by~\Cref{4sen}, the joint vector $(Z, z)$ is has 
  sensitivity bounded by 4 with respect to $\ell_1$ norm. Therefore,
  the mechanism satisfies $\eps$-differential privacy by the privacy
  guarantee of the Laplace mechanism.

  Let $M = \sqrt{1/\lambda}$ and
  $L_p(\theta) = \frac{1}{2n} \left( - 2 \langle z, \theta\rangle
  \right) + \frac{1}{2n} (\theta^\intercal Z \theta) + \frac{\lambda
    \|\theta\|_2^2}{2}$. Observe that
  $\theta_p = \argmin_{\theta\in C} L_p(\theta)$.  
  Our goal is to bound $L(\theta_p) - L(\theta^*)$, which can be
  written as follows
\begin{align*}
  L(\theta_p) - L(\theta^*) &= L(\theta_p) - L_p(\theta_p) + L_p(\theta_p) - L_p(\theta^*) + L_p(\theta^*) - L(\theta^*) \\
                            &\leq L(\theta_p) - L_p(\theta_p) + L_p(\theta^*) - L(\theta^*)\\
                            &= \frac{1}{2n}\left(2 \langle b, \theta_p \rangle - \theta_p^\intercal B \theta_p  \right) -  \frac{1}{2n}\left( 2 \langle b, \theta^* \rangle - (\theta^*)^\intercal B \theta^* \right)
\end{align*}

Moreover,
$\langle b, \theta_p\rangle \leq \|b\|_2 \|\theta_p\|_2 \leq M
\|b\|_2$ and 
\begin{align*}
-  \theta_p^\intercal B \theta_p &= - \sum_{(s, t)\in [p]^2} B_{st} (\theta_p)_s (\theta_p)_t \\
                                &\leq \left( \sum_{(s, t)} B_{st}^2\right)^{1/2} \left( \sum_{s, t}(\theta_p)_s^2 (\theta_p)_t^2\right)^{1/2}\\
                                &= \|B\|_F \left[\left( \sum_s (\theta_p)_s^2 \right)^2\right]^{1/2} \\
                                	&\leq \|B\|_F M^2
\end{align*}
By~\Cref{jensen}, we also have $\Ex{}{\|B\|_F} \leq 4\sqrt{2}p/\eps$
and $\Ex{}{\|b\|_2} \leq 4\sqrt{2p}/\eps$. Finally, 
\begin{align*}
  \Ex{}{L(\theta_p) - L(\theta^*)} &\leq \Ex{}{\frac{1}{2n}\left(2 \langle b, \theta_p \rangle - \theta_p^\intercal B \theta_p  \right) -  \frac{1}{2n}\left( 2 \langle b, \theta^* \rangle - (\theta^*)^\intercal B \theta^* \right)}\\
                                   &= \Ex{}{\frac{2\langle b, \theta_p \rangle - \theta_p^\intercal B \theta_p}{2n}}\\
                                   &\leq \frac{\Ex{}{2M\|b\|_2} + \Ex{}{M^2 \|B\|_F}}{2n} \\
                                   &\leq \frac{4\sqrt{2}(2\sqrt{p}M + p M^2)}{n \eps} 
\end{align*}
which recovers our stated bound.
\end{proof}

Next, we will also provide a theoretical result for applying output
perturbation (with Laplace noise) to the ridge regression
problem. This will provides us the ``theory curve'' for output
perturbation in ridge regression plot of
\Cref{subfig:privacy-twitter-ridge-theory}.

First, the following sensitivity bound on the optimal solution for $L$
follows directly from the strong convexity of $L$.

\begin{lemma}\label{1sen}
  Let $C$ be a bounded convex set in $\RR^p$ with $\|C\|_2\leq M$. Let
  $D$ and $D'$ be a pair of neighboring datasets, and let
  $\theta^* = \argmin_{\theta\in C} L(\theta, D)$ and
  $\theta^\bl =\argmin_{\theta\in C} L(\theta, D')$. Then
  $\|\theta^* - \theta^\bl\|_1 \leq (M+1) \sqrt{\frac{p}{n\lambda}}$.
\end{lemma}

\begin{theorem}
  Let $\eps > 0$ and $C$ be a bounded convex set with $\|C\|_2\leq
  \sqrt{1/\lambda}$. Let $r = (\sqrt{1/\lambda}+1) \sqrt{p/(n\lambda)}/\eps$.  Consider the
  following mechanism $\cM$ that for any input dataset $D$ first
  computes the optimal solution
  $\theta^* = \argmin_{\theta\in C} L(\theta)$, and then outputs
  $\theta_p = \theta^* + b$, where $b$ is a random vector with its
  entries drawn i.i.d. from $\Lap{r}$. Then $\cM$ satisfies
  $\eps$-differential privacy, and $\theta_p$ satisfies
  \[
    \Ex{b}{L(\theta_p) - L(\theta^*)} \leq = \left( \frac{1}{n} +
      \lambda \right) \frac{(\sqrt{1/\lambda}+1)^2 p^2}{n\lambda\eps^2}.
  \]
\end{theorem}

\begin{proof}
  The privacy guarantee follows directly from the use of Laplace
  mechanism and the $\ell_1$ sensitivity bound in~\Cref{1sen}.

  For each data point $d_i = (X_i, y_i)$, we have
\begin{align*}
  (y_i - \langle \theta_p, X_i\rangle)^2 - (y_i - \langle \theta^*, X_i\rangle)^2
  &= (\langle \theta_p, X_i \rangle )^2 - (\langle \theta^*, X_i \rangle )^2 - 2 \langle b, X_i \rangle \\
  &= b^\intercal (X_i^\intercal X_i) b + (\theta^*)^\intercal  (X_i^\intercal X_i) b + b^\intercal (X_i^\intercal X_i)\theta^* - 2 \langle b, X_i \rangle
\end{align*}
Since each entry in $b$ has mean 0, we can simplify the expectation as
\begin{align*}
  \Ex{}{(y_i - \langle \theta_p, X_i\rangle)^2 - (y_i - \langle
  \theta^*, X_i\rangle)^2} &= \Ex{}{b^\intercal (X_i^\intercal X_i) b}\\
  &= \Ex{}{\left(\langle b, X_i \rangle\right)^2 }\\
  &\leq \Ex{}{\|b\|_2^2 \|X_i\|_2^2} \\
  &= \Ex{}{\|b\|_2^2}\Ex{}{\|X_i\|_2^2}\\
  &\leq \Ex{}{\|b\|_2^2} \leq 2p r^2
\end{align*}

In the following, let $M = \sqrt{1/\lambda}$. We can then bound
\begin{align*}
  \|\theta_p\|_2^2 - \| \theta^*\|_2^2 &= \sum_{s\in [p]} \left[(\theta_s + b_s)^2 - \theta_s^2\right]\\
  &= \sum_{s\in [p]} \left[2\theta_s b_s + b_s^2\right],
\end{align*}
Again, since each $b_s$ is drawn from $\Lap{r}$, we get
\begin{align*}
  \Ex{}{\|\theta_p\|_2^2 - \| \theta^*\|_2^2} &= \Ex{}{\sum_s b_s^2}\\
                                              &= \sum_s\Ex{}{b_s^2} = 2p r^2.
\end{align*}
To put all the pieces together and plugging in the value of $r$, we
get
\begin{align*}
  \Ex{b}{L(\theta_p) - L(\theta^*)} \leq \left(\frac{1}{2n} + \frac{\lambda}{2}  \right) 2pr^2\\
  = \left( \frac{1}{n} + \lambda \right) \frac{(M+1)^2 p^2}{n\lambda\eps^2}
\end{align*}
which recovers our stated bound.
\end{proof}

\subsection{Logistic Regression}

In this subsection, the input data $D$ consists of $n$ labelled
examples $(X_1 , y_1) ,\ldots , (X_n, y_n)$, such that for each $i$,
$x_i \in \RR^p$, $\|x_i\|_1 \leq 1$, and $y_i\in \{-1, 1\}$.

We consider the logistic loss function:
$\ell(\theta, (X_i, y_i)) = \log(1 + \exp(-y_i\theta^\intercal X_i))$,
and our empirical loss is defined as
\[
  L(\theta, D) = \frac{1}{n} \sum_{i=1}^n \log(1 + \exp(- y_i
  \theta^\intercal X_i)) + \frac{\lambda\|\theta\|_2^2}{2}.
\]

In output perturbation, the noise needs to scale with the
$\ell_1$-sensitivity of the optimal solution, which is given by the
following lemma.

\begin{lemma}\label{ahoy}
  Let $D$ and $D'$ be a pair of neighboring datasets. Let
  $\theta = \argmin_{w\in \RR^p} L(w, D)$ and
  $\theta' = \argmin_{w'\in \RR^p} L(w', D')$. Then
  $\|\theta - \theta'\|_1 \leq \frac{2\sqrt{p}}{n\lambda}$.
\end{lemma}

\begin{proof}[Proof of~\Cref{ahoy}]
By Corollary 8 of~\cite{CMS11}, we can bound
\[
  \|\theta - \theta'\|_2 \leq \frac{2}{n \lambda}
\]
By the fact that $\|a\|_1 \leq \sqrt{p}\|a\|_2$ for any $a\in \RR^p$,
we recover the stated result.
\end{proof}

We will show that the optimal solution for the unconstrained problem
has $\ell_2$ norm no more than $\sqrt{2\log 2/\lambda}$.

\begin{claim}
  The (unconstrained) optimal solution $\theta^*$ has norm
  $\|\theta^*\|_2 \leq \sqrt{\frac{2\log 2}{\lambda}}$.
\end{claim}

\begin{proof}
  Note that the weight vector $\theta = \vec{0}$ has loss
  $\log{2}$. Therefore, $L(\theta^*) \leq \log 2$. Since the logistic
  loss is positive, we know that the regularization term 
  \[
    \frac{\lambda}{2} \|\theta^*\|_2^2 \leq \log 2.
  \]
  It follows that $\|\theta^*\|_2 \leq \sqrt{\frac{2\log 2}{\lambda}}$.
\end{proof}

We will focus on generating hypotheses $\theta$ within the set
$C = \{a \in \RR^p \mid \|a\|_2\leq \sqrt{2\log 2/\lambda}\}$. Then we
can bound the $\ell_1$ sensitivity of the excess risk using the
following result.

\begin{claim}\label{3sen-logist}
  Let $D$ and $D'$ be a pair of neighboring datasets. Then for any
  $\theta\in \RR^p$ such that $\|\theta\|_2\leq M$,
\[
  |L(\theta, D) - L(\theta, D')| \leq \frac{2}{n} \log\left(\frac{1 +
      \exp(M)}{1 + \exp(-M)} \right)
\]
\end{claim}

The following fact is useful for our utility analysis for the output
perturbation method.

\begin{claim}
  Fix any data point $(x, y)$ such that $\|x\|_1\leq 1$ and
  $y\in \{-1, 1\}$. The logistic loss function $\ell(\theta, (x, y))$
  is a 1-Lipschitz function in $\theta$.
\end{claim}

\begin{proof}[Proof of~\Cref{logistman}]
  The privacy guarantee follows directly from the use of Laplace
  mechanism and the $\ell_1$-sensitivity bound in~\Cref{ahoy}. Since
  the logistic loss function is 1-Lipschitz. For any $(x, y)$ in our
  domain,
  \[
    |\ell(\theta^*, (x,y)) - \ell(\theta^p, (x, y)) | \leq \|\theta^*
    - \theta^p\|_2 = \|b\|_2.
  \]
  Furthermore, 
  \[
    { \|\theta_p\|_2^2 - \|\theta^*\|_2^2} = { \|\theta^* + b\|_2^2 -
      \|\theta^*\|_2^2} = 2\langle b, \theta^* \rangle + \|b\|_2^2
  \]
  By~\Cref{jensen} and the property of the Laplace distribution, we
  know that
\[
  \Ex{}{\|b\|_2} \leq \sqrt{2p}r \qquad \mbox{ and }\qquad 
  \Ex{}{\|b\|_2^2} = 2p r^2.
\]
It follows that
\begin{align*}
  \Ex{b}{L(\theta_p) - L(\theta^*)} &\leq \Ex{b}{\|b\|_2} + \frac{\lambda}{2}\Ex{}{\|b\|_2^2} \\
                                    &\leq \sqrt{2p} r + p\lambda r^2 = \frac{2\sqrt{2}pr}{n\lambda \eps} + \frac{4p^2}{n^2\lambda\eps^2},
\end{align*}
which recovers the stated bound.
\end{proof}

We include the full details of $\Onoise$ in \Cref{alg:onoise}.

\begin{algorithm}[h]
  \caption{Output Perturbation with Noise-Reduction:
    $\Onoise(D,  \{\eps_1, \ldots, \eps_T\}, \alpha, \gamma)$}
 \label{alg:onoise}
  \begin{algorithmic}
    \State{\textbf{Input:} private data set $D = (X, y)$, accuracy
      parameter $\alpha$, 
      privacy levels $\eps_1 < \eps_2 < \ldots < \eps_T$, and failure
      probability $\gamma$}

    \State{Let $M = \sqrt{2\log 2/\lambda}$}

    \State{Instantiate Interactive AboveThreshold:
      $\cA = (D, \eps_0, \alpha/2, 2\log{(1+\exp(M))/(1+\exp(-M))}/(n),
      \cdot)$
      with $\eps_0 = 16\Delta (\log(2T/\gamma))/ \alpha$ and
      $\Delta = 2\log{(1+\exp(M))/(1+\exp(-M))}/(n)$ }

    \State{Let $C = \{a\in \RR^p \mid \|a\|_2\leq \sqrt{1/\lambda}\}$
      and $\theta^* = \argmin_{\theta\in \RR^p} L(\theta)$}
    \State{Generate hypotheses: $\{\theta^t\} = \NR(\theta^*, \frac{2\sqrt{p}}{n\lambda}, \{\eps_1, \ldots, \eps_T\})$}
    \For{$t = 1, \ldots , T$:}

    \If{ $\|\theta^t\|_2 \leq M$}{ Set $\theta^t = M( \theta^t /\|\theta^t\|_2)$}  \Comment{Rescale the norm for bounded sensitivity} \EndIf
    \State{ Let      $f^t(D) = L(D, \theta^*) - L(D, \theta^t)$   }

    \State{Query $\cA$ with $f^t$}

    \If{ yes}{ \textbf{Output} $(t, \theta^t)$} 
\EndIf
      \EndFor

      \State{\textbf{Output:} $(\perp, \theta^*)$}
    \end{algorithmic}
  \end{algorithm}

\section{Experiments} \label{app:experiments}

\subsection{Parameters and data}
For simplicity and to avoid over-fitting, we fixed the following parameters for both experiments:
\begin{itemize}
  \item $n=$100,000 (number of data points)
  \item $\lambda = 0.005$ (regularization parameter)
  \item $\gamma = 0.10$ (requested failure probability)
  \item $\epsilon_1 = 4E$, where $E$ is the inversion of the theory guarantee for the underlying algorithm.
        For example in the logistic regression setting where the algorithm is Output Perturbation, $E$ is the value such that setting $\eps=E$ guarantees \textbf{expected} excess risk of at most $\alpha$.
  \item $\epsilon_T = 1.0/n$.
  \item $\alpha = 0.005, 0.010, 0.015, \dots, 0.200$ (requested excess error bound).
\end{itemize}
For NoiseReduction, we choose $T = 1000$ (maximum number of iterations) and set $\epsilon_t = \epsilon_1 r^t$ for the appropriate $r$, i.e. $r = \left(\frac{\epsilon_T}{\epsilon_1}\right)^{1/T}$.

For the Doubling method, $T$ is equal to the number of doubling steps until $\epsilon_t$ exceeds $\epsilon_T$, i.e. $T = \lceil \log_2(\epsilon_1/\epsilon_T) \rceil$.

\paragraph{Features, labels, and transformations.}
The Twitter dataset has $p=77$ features (dimension of each $x$), relating to measurements of activity relating to a posting; the label $y$ is a measurement of the ``buzz'' or success of the posting.
Because general experience suggests that such numbers likely follow a heavy-tailed distribution, we transformed the labels by $y \mapsto \log(1+y)$ and set the taks of predicting the transformed label.

The KDD-99 Cup dataset has $p=38$ features relating to attributes of a network connection such as duration of connection, number of bytes sent in each direction, binary attributes, etc.
The goal is to classify connections as innocent or malicious, with malicious connections broken down into further subcategories.
We transformed three attributes containing likely heavy-tailed data (the first three mentioned above) by $x_i \mapsto \log(1+x_i)$, dropped three columns containing textual categorical data, and transformed the labels into $1$ for any kind of malicious connection and $0$ for an innocent one.
(The feature length $p=38$ is after dropping the text columns.)

For both datasets, we transformed the data by renormalizing to maximum $L1$-norm $1$.
That is, we computed $M = \max_i \|x_i\|_1$, and transformed each $x_i \mapsto x_i/M$.
In the case of the Twitter dataset, we did the same (separately) for the $y$ labels.
This is \emph{not} a private operation (unlike the previous ones) on the data, as it depends precisely on the maximum norm.
We do not consider the problem of privately ensuring bounded-norm data, as it is orthogonal to the questions we study.

The code for the experiments is implemented in python3 using the numpy and scikit-learn libraries.

\subsection{Additional results}
Figure \ref{fig:accuracies} plots the empirical accuracies of the output hypotheses, to ensure that the algorithms are achieving their theoretical guarantees.
In fact, they do significantly better, which is reasonable considering the private testing methodology: set a threshold significantly below the goal $\alpha$, add independent noise to each query, and accept only if the query plus noise is smaller than the threshold.
Combined with the requirement to use tail bounds, the accuracies tend to be significantly smaller than $\alpha$ and with significantly higher probability than $1-\gamma$.
(Recall: this is not necessarily a good thing, as it probably costs a significant amount of extra privacy.)

Figure \ref{fig:privacy-breakdown} shows the breakdown in privacy losses between the ``privacy test'' and the ``hypothesis generator''.
In the case of NoiseReduction, these are AboveThreshold's $\eps_A$ and the $\eps_t$ of the private method, Covariance Perturbation or Output Perturbation.
In the case of Doubling, these are the accrued $\eps$ due to tests at each step and due to Covariance Perturbation or Output Perturbation for outputting the hypotheses.

This shows the majority of the privacy loss is due to testing for privacy levels.
One reason why might be that the cost of privacy tests depends heavily on certain constants, such as the norm of the hypothesis being tested.
This norm is upper-bounded by a theoretical maximum which is used, but a smaller maximum would allow for significantly higher computed privacy levels for the same algorithm.
In other words, the analysis might be loose compared to an analysis that knows the norms of the hypotheses, although this is a private quantity.
Figure \ref{fig:hypothesis-norms} supports the conclusion that generally, the theoretical maximum was very pessimistic in our cases.
Note that a tenfold reduction in norm gives a tenfold reduction in privacy level for logistic regression, where sensitivity is linear in maximum norm; and a \emph{hundred-fold} reduction for ridge regression.

\begin{figure}
  \begin{subfigure}{0.48\linewidth}
    \includegraphics[width=\linewidth]{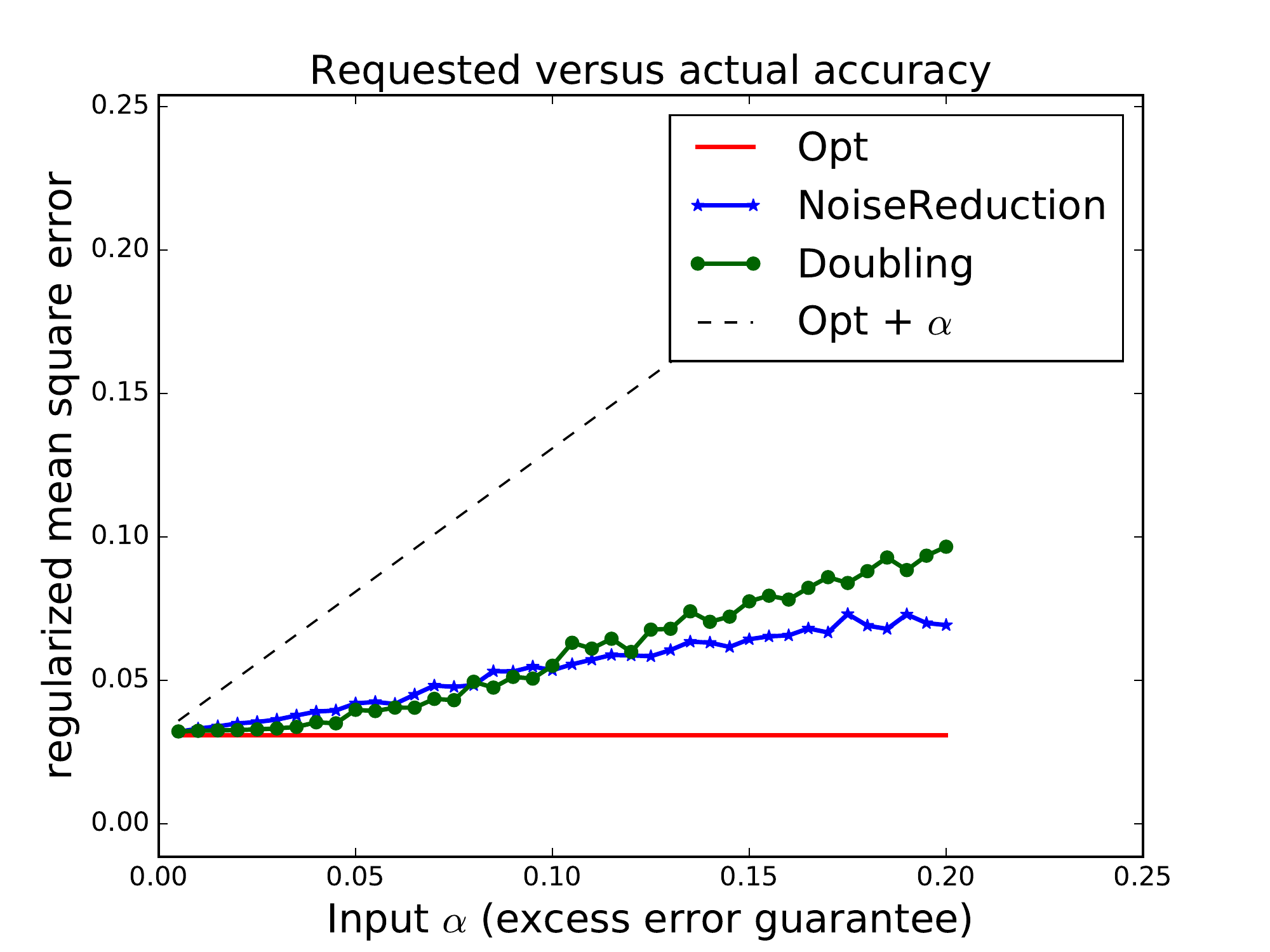}
    \caption{\textbf{Linear (ridge) regression.}}
  \end{subfigure}
  \begin{subfigure}{0.48\linewidth}
    \includegraphics[width=\linewidth]{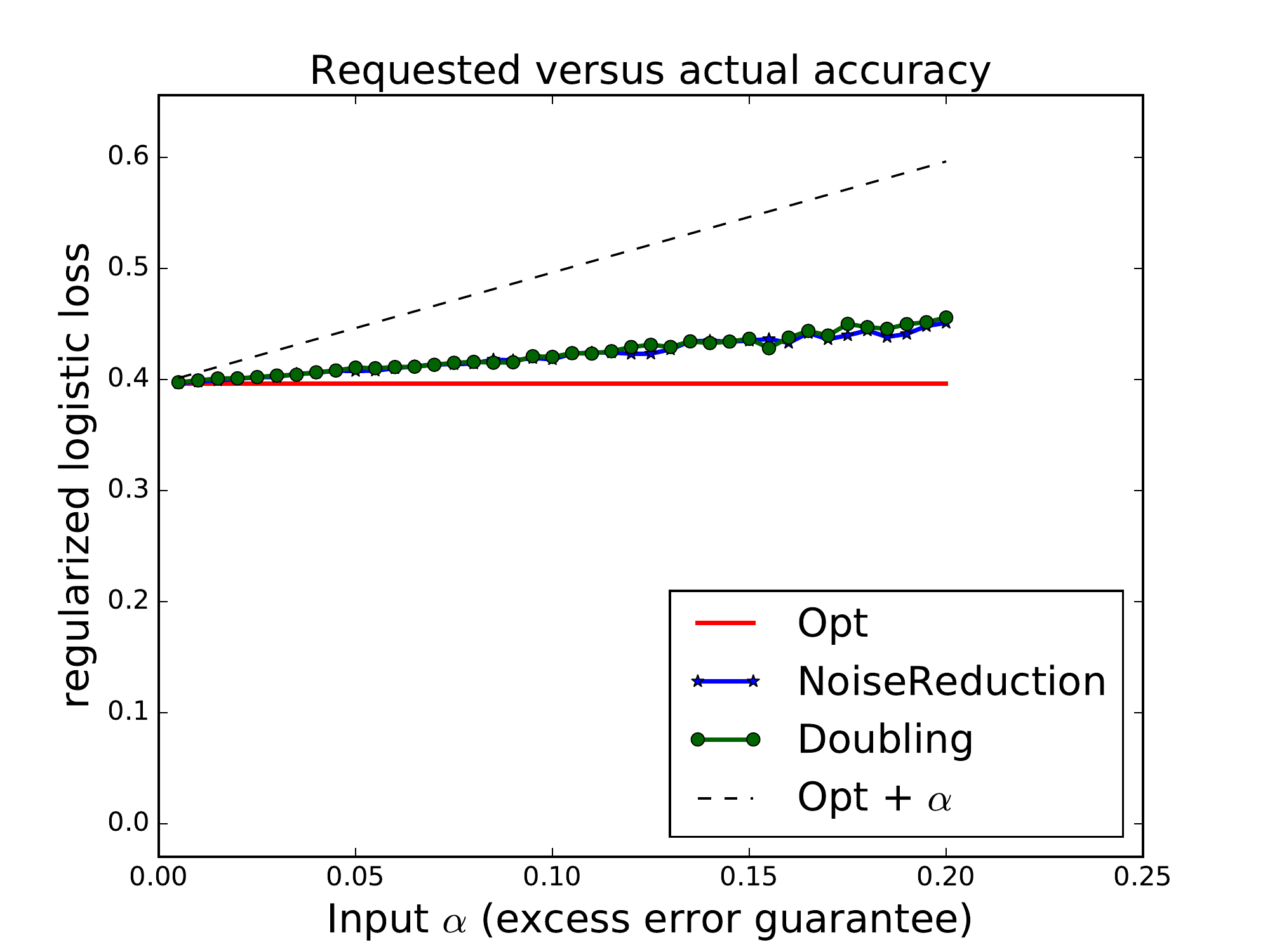}
    \caption{\textbf{Regularized logistic regression.}}
  \end{subfigure}

  \begin{subfigure}{0.48\linewidth}
    \includegraphics[width=\linewidth]{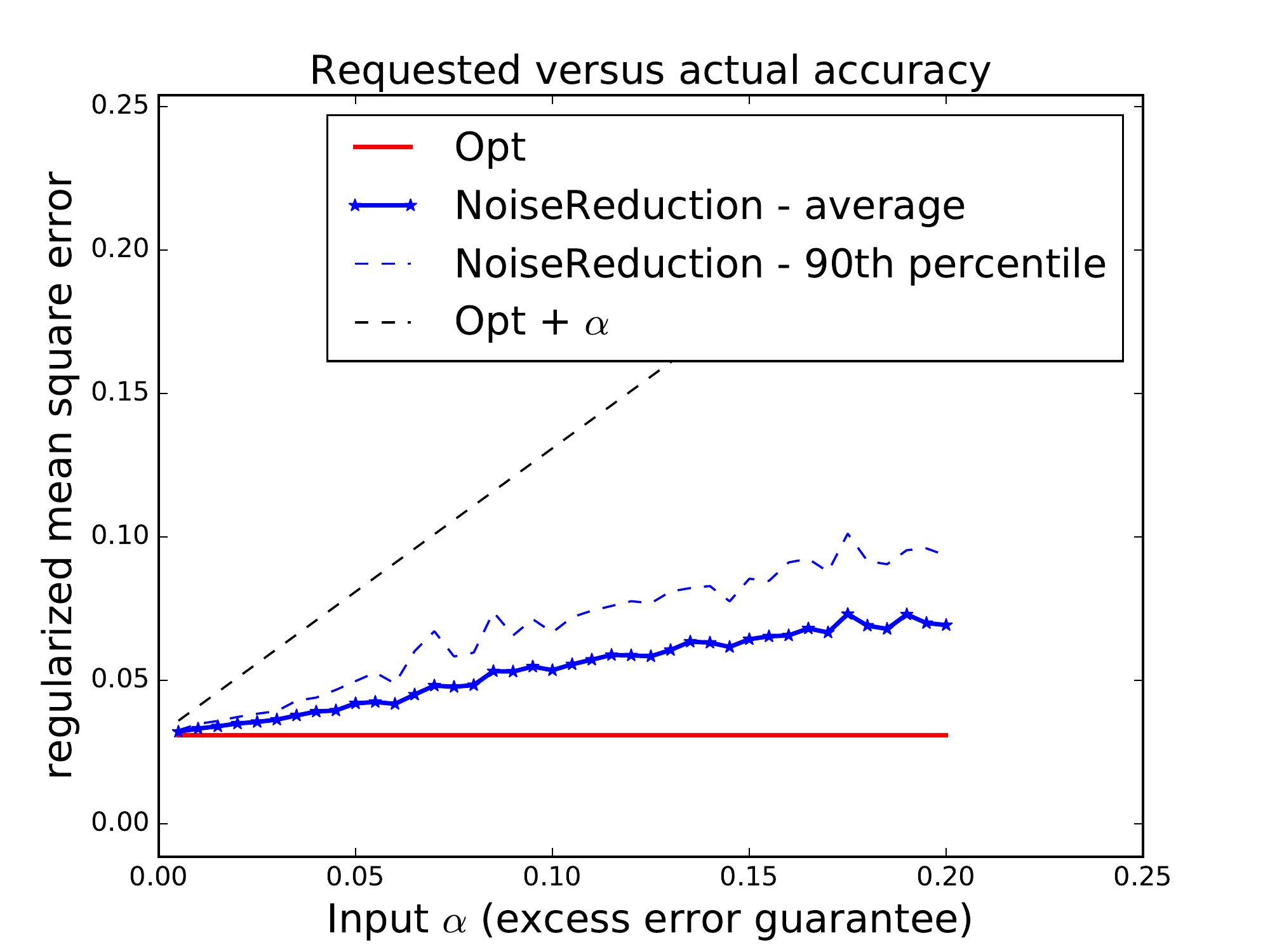}
    \caption{\textbf{Linear (ridge) regression.}}
  \end{subfigure}
  \begin{subfigure}{0.48\linewidth}
    \includegraphics[width=\linewidth]{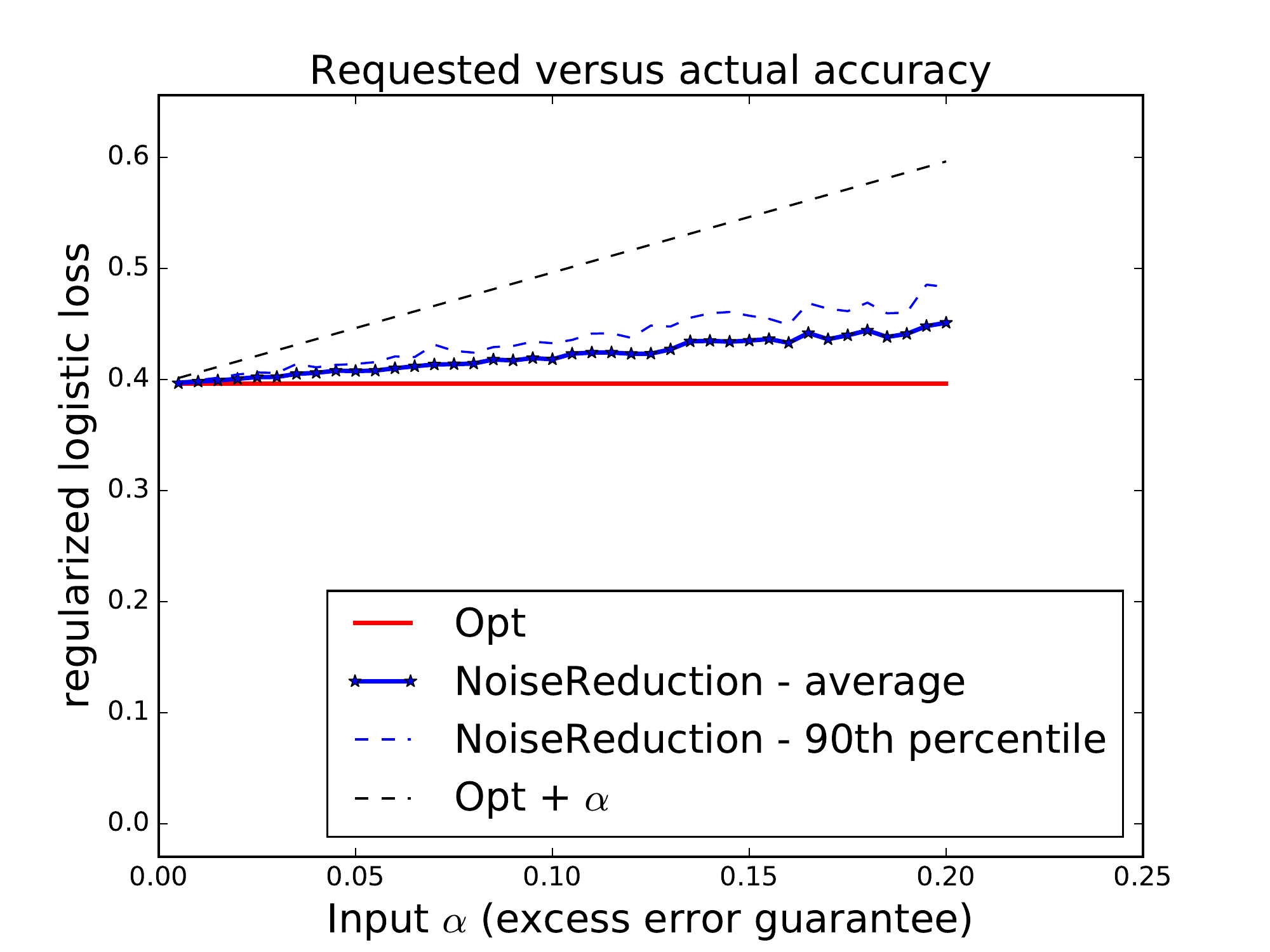}
    \caption{\textbf{Regularized logistic regression.}}
  \end{subfigure}
  \caption{\textbf{Empirical accuracies.} The dashed line shows the requested accuracy level, while the others plot the actual accuracy achieved.
           Due most likely due to a pessimistic analysis and the need to set a small testing threshold, accuracies are significantly better than requested for both methods.}
  \label{fig:accuracies}
\end{figure}

\begin{figure}
  \begin{subfigure}{0.48\linewidth}
    \includegraphics[width=\linewidth]{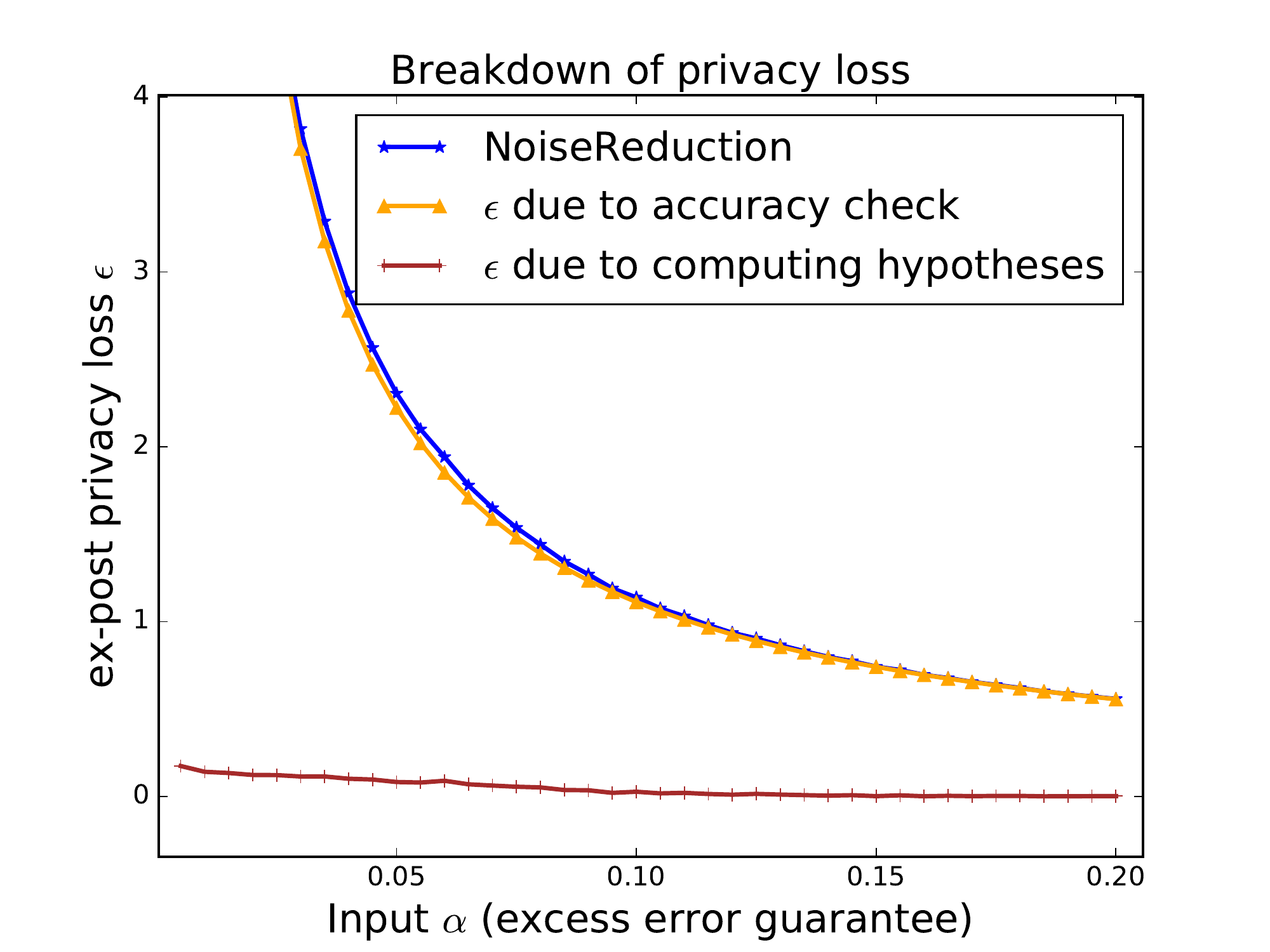}
    \caption{\textbf{Linear (ridge) regression.}}
  \end{subfigure}
  \begin{subfigure}{0.48\linewidth}
    \includegraphics[width=\linewidth]{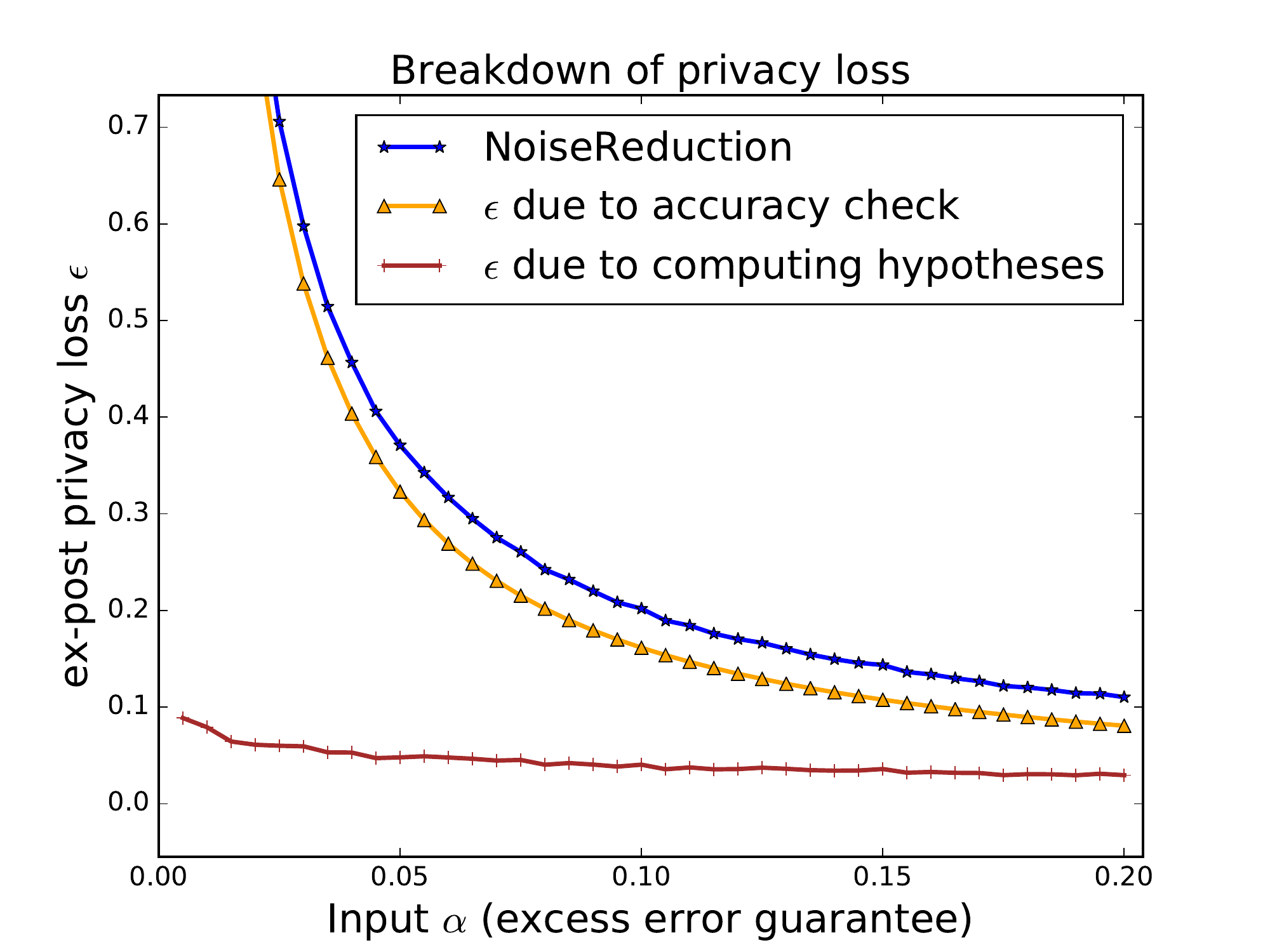}
    \caption{\textbf{Regularized logistic regression.}}
  \end{subfigure}
  \caption{\textbf{Privacy breakdowns.} Shows the amount of empirical privacy loss due to computing the hypotheses themselves and the losses due to testing their accuracies.}
  \label{fig:privacy-breakdown}
\end{figure}

\begin{figure}
  \begin{subfigure}{0.48\linewidth}
    \includegraphics[width=\linewidth]{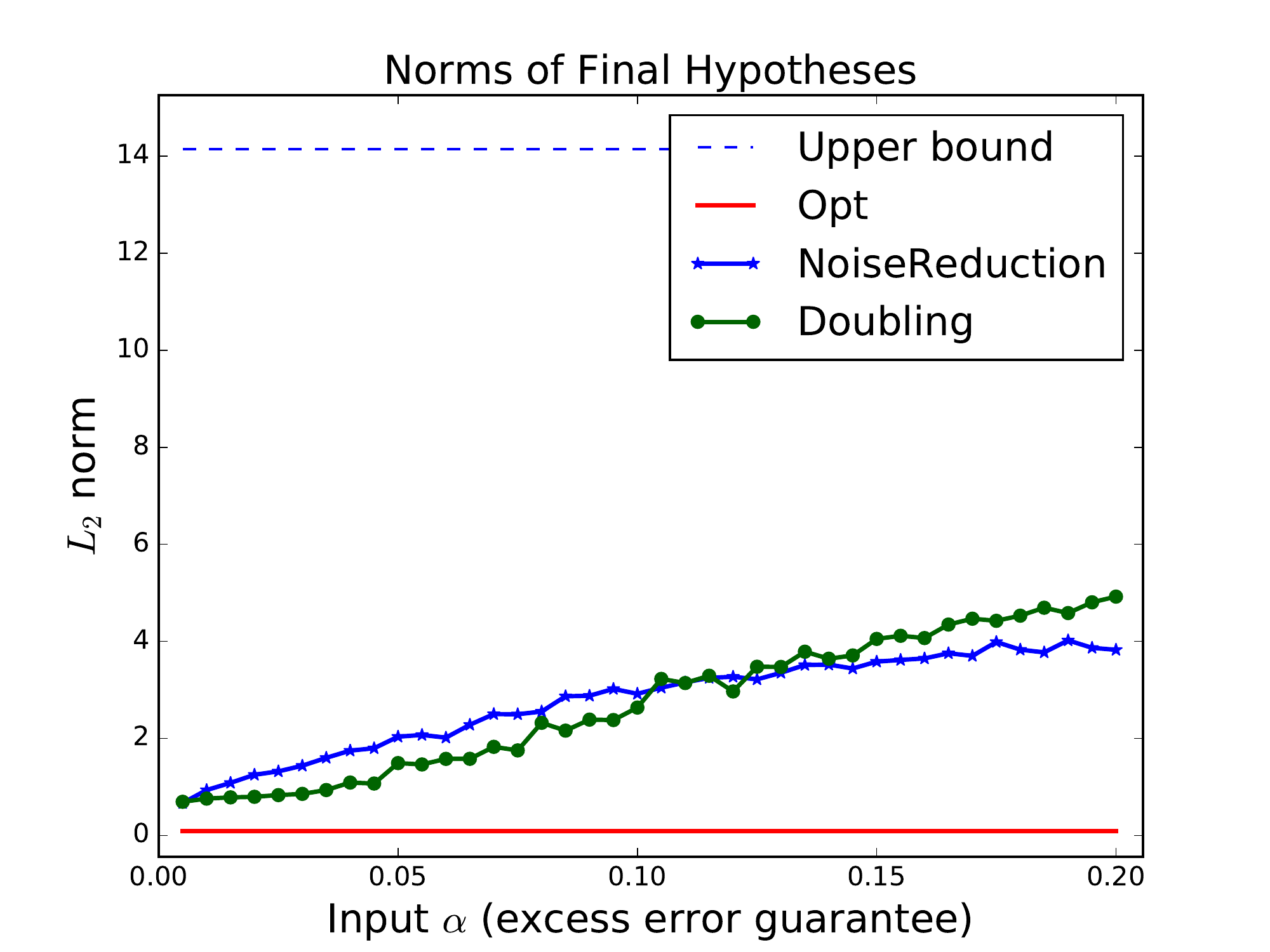}
    \caption{\textbf{Linear (ridge) regression.}}
  \end{subfigure}
  \begin{subfigure}{0.48\linewidth}
    \includegraphics[width=\linewidth]{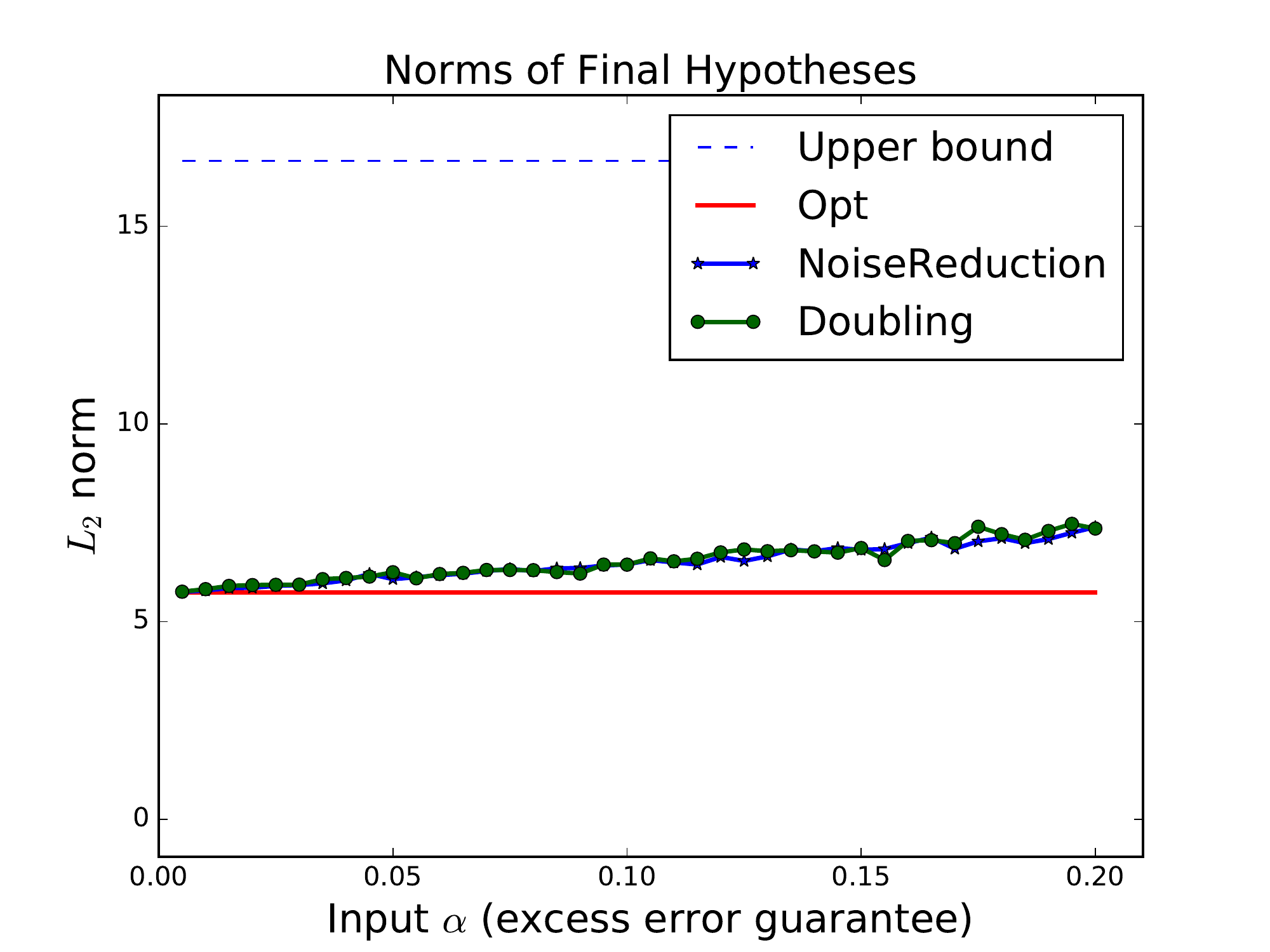}
    \caption{\textbf{Regularized logistic regression.}}
  \end{subfigure}
  \caption{\textbf{$L_2$ norms of final hypotheses.} Shows the average $L_2$ norm of the output $\hat{\theta}$ for each method, versus the theoretical maximum of $1/\sqrt{\lambda}$ in the case of ridge regression and $\sqrt{2\log(2)/\lambda}$ in the case of regularized logistic regression.}
  \label{fig:hypothesis-norms}
\end{figure}

\subsection{Supporting theory}
\bo{This is fuzzy for randomized algs.}
\begin{claim} \label{claim:doubling-2-opt}
  For the ``doubling method'', the factor $2$ increase in $\eps$ at each time step gives the optimal worst case ex post privacy loss guarantee.
\end{claim}
\begin{proof}
  In a given setting, suppose $\eps^*$ is the ``final'' level of privacy at which the algorithm would halt.
  With a factor $1/r$ increase for $r<1$, the final loss may be as large as $\eps^*/r$.
  The total loss is the sum of that loss and all previous losses, i.e. if $t$ steps were taken:
  \begin{align*}
    (\eps^*/r) + r\cdot(\eps^*/r) + \dots + r^{t-1} \cdot (\eps^*/r)
      &= (\eps^*/r) \sum_{j=0}^{t-1} r^j  \\
      &\to (\eps^*/r) \sum_{j=0}^{\infty} r^j  \\
      &= \frac{\eps^*}{r(1-r)}  \\
      &\geq 4\eps^* .
  \end{align*}
  The final inequality implies that setting $r=0.5$ and $(1/r) = 2$ is optimal.
  The asymptotic $\to$ is justified by noting that the starting $\eps_1$ may be chosen arbitrarily small, so there exist parameters that exceed the value of that summation for any finite $t$; and the summation limits to $\frac{1}{1-r}$ as $t \to \infty$.
\end{proof}

\bo{Everything below is low-priority / optional for this submission. They just give the derivation of some constants used in the code.}

%

\bo{Details about the naive method's privacy test and its $1-\gamma$ guarantee}

\bo{SGD theory curves details}


\end{document}